\def\isarxiv{1} 

\ifdefined\isarxiv
\documentclass[11pt]{article}

\usepackage[numbers]{natbib}

\else
\documentclass{article}
\usepackage[submission]{colm2025_conference}
\fi

\ifdefined\isarxiv
\usepackage{amsmath}
\usepackage{amsthm}
\usepackage{amssymb}
\usepackage{algorithm}
\usepackage{subfig}
\usepackage{algpseudocode}
\usepackage{graphicx}
\usepackage{grffile}
\usepackage{wrapfig,epsfig}
\usepackage{url}
\usepackage{xcolor}
\usepackage{epstopdf}
\usepackage{bbm}
\usepackage{dsfont}

\else
\usepackage{amsmath}
\usepackage{amsthm}
\usepackage{amssymb}
\usepackage{algorithm}
\usepackage{subfig}
\usepackage{algpseudocode}
\usepackage{graphicx}
\usepackage{grffile}
\usepackage{wrapfig,epsfig}
\usepackage{microtype}
\usepackage{hyperref}
\usepackage{url}
\usepackage{xcolor}
\usepackage{epstopdf}
\usepackage{bbm}
\usepackage{dsfont}
\usepackage{booktabs}
\usepackage{lineno}
\fi

\allowdisplaybreaks

\ifdefined\isarxiv

\usepackage{tikz}
\usepackage{hyperref}  
\hypersetup{colorlinks=true,citecolor=blue,linkcolor=blue} 
\usetikzlibrary{arrows}
\usepackage[margin=1in]{geometry}

\else

\usepackage{microtype}
\usepackage{hyperref}
\definecolor{darkblue}{rgb}{0, 0, 0.5}
\hypersetup{colorlinks=true, citecolor=darkblue, linkcolor=darkblue, urlcolor=darkblue}

\fi
 
\graphicspath{{./figs/}}

\theoremstyle{plain}
\newtheorem{theorem}{Theorem}[section]
\newtheorem{lemma}[theorem]{Lemma}
\newtheorem{definition}[theorem]{Definition}

\newtheorem{corollary}[theorem]{Corollary}

\newtheorem{fact}[theorem]{Fact}
\newtheorem{remark}[theorem]{Remark}

\newcommand{\wh}{\widehat}
\newcommand{\wt}{\widetilde}
\newcommand{\ov}{\overline}

\newcommand{\R}{\mathbb{R}}

\DeclareMathOperator*{\E}{{\mathbb{E}}}

\DeclareMathOperator*{\Z}{\mathbb{Z}}

\DeclareMathOperator{\poly}{poly}

\DeclareMathOperator{\diag}{diag}

\newcommand{\x}{\mathsf{x}} 
\renewcommand{\S}{\mathcal{S}} 
\newcommand{\maj}{\mathsf{MAJ}} 

\makeatletter
\newcommand*{\RN}[1]{\expandafter\@slowromancap\romannumeral #1@}
\makeatother

\ifdefined\isarxiv 
\renewcommand{\citet}{\cite}
\else 
\renewcommand{\cite}{\citep}
\fi

\usepackage{lineno}

\begin{document}

\ifdefined\isarxiv

\date{}

\title{Provable Failure of Language Models in Learning Majority Boolean Logic via Gradient Descent}
\author{
Bo Chen\thanks{\texttt{ bc7b@mtmail.mtsu.edu}. Middle Tennessee State University.}
\and
Zhenmei Shi\thanks{\texttt{
zhmeishi@cs.wisc.edu}. University of Wisconsin-Madison.}
\and
Zhao Song\thanks{\texttt{ magic.linuxkde@gmail.com}. The Simons Institute for the Theory of Computing at the UC, Berkeley.}
\and
Jiahao Zhang\thanks{\texttt{ ml.jiahaozhang02@gmail.com} Independent Researcher.}
}

\else

\title{Provable Failure of Language Models in Learning Simple Boolean Logic via Gradient Descent} 

\author{Antiquus S.~Hippocampus, Natalia Cerebro \& Amelie P. Amygdale \thanks{ Use footnote for providing further information
about author (webpage, alternative address)---\emph{not} for acknowledging
funding agencies.  Funding acknowledgements go at the end of the paper.} \\
Department of Computer Science\\
Cranberry-Lemon University\\
Pittsburgh, PA 15213, USA \\
\texttt{\{hippo,brain,jen\}@cs.cranberry-lemon.edu} \\
\And
Ji Q. Ren \& Yevgeny LeNet \\
Department of Computational Neuroscience \\
University of the Witwatersrand \\
Joburg, South Africa \\
\texttt{\{robot,net\}@wits.ac.za} \\
\AND
Coauthor \\
Affiliation \\
Address \\
\texttt{email}
}

\ifdefined\colmsubmissiontrue
\linenumbers
\fi

\maketitle

\fi

\ifdefined\isarxiv
\begin{titlepage}
  \maketitle
  \begin{abstract}
Recent advancements in Transformer-based architectures have led to impressive breakthroughs in natural language processing tasks, with models such as GPT-4, Claude, and Gemini demonstrating human-level reasoning abilities. However, despite their high performance, concerns remain about the inherent limitations of these models, especially when it comes to learning basic logical functions. While complexity-theoretic analyses indicate that Transformers can represent simple logic functions (e.g., $\mathsf{AND}$, $\mathsf{OR}$, and majority gates) by its nature of belonging to the $\mathsf{TC}^0$ class, these results assume ideal parameter settings and do not account for the constraints imposed by gradient descent-based training methods. In this work, we investigate whether Transformers can truly learn simple \textbf{majority} functions when trained using gradient-based methods. We focus on a simplified variant of the Transformer architecture and consider both $n=\mathrm{poly}(d)$ and $n=\exp(\Omega(d))$ number of training samples, where each sample is a $d$-size binary string paired with the output of a basic majority function. Our analysis demonstrates that even after $\mathrm{poly}(d)$ gradient queries, the generalization error of the Transformer model still remains substantially large, growing exponentially with $d$. This work highlights fundamental optimization challenges in training Transformers for the simplest logical reasoning tasks and provides new insights into their theoretical limitations. 

  \end{abstract}
  \thispagestyle{empty}
\end{titlepage}

{\hypersetup{linkcolor=black}
\tableofcontents
}
\newpage

\else

\begin{abstract}
Recent advancements in Transformer-based architectures have led to impressive breakthroughs in natural language processing tasks, with models such as GPT-4, Claude, and Gemini demonstrating human-level reasoning abilities. However, despite their high performance, concerns remain about the inherent limitations of these models, especially when it comes to learning basic logical functions. While complexity-theoretic analyses indicate that Transformers can represent simple logic functions (e.g., $\mathsf{AND}$, $\mathsf{OR}$, and majority gates) by its nature of belonging to the $\mathsf{TC}^0$ class, these results assume ideal parameter settings and do not account for the constraints imposed by gradient descent-based training methods. In this work, we investigate whether Transformers can truly learn simple \textbf{majority} functions when trained using gradient-based methods. We focus on a simplified variant of the Transformer architecture and consider both $n=\mathrm{poly}(d)$ and $n=\exp(\Omega(d))$ number of training samples, where each sample is a $d$-size binary string paired with the output of a basic majority function. Our analysis demonstrates that even after $\mathrm{poly}(d)$ gradient queries, the generalization error of the Transformer model still remains substantially large, growing exponentially with $d$. This work highlights fundamental optimization challenges in training Transformers for the simplest logical reasoning tasks and provides new insights into their theoretical limitations. 
\end{abstract}

\fi

\section{Introduction}

In recent years, Transformer-based architectures~\cite{vsp+17} have achieved unprecedented breakthroughs in a wide range of natural language processing tasks, including machine translation~\cite{cfb+18,ajf19,has+23}, text summarization~\cite{esrm21,vvb+24}, and question answering~\cite{lly+23}. Representative models such as GPT-4~\cite{o23}, Claude~\cite{c24}, and Gemini~\cite{g24} have not only pushed benchmark performance to new heights but have also exhibited remarkable human-level reasoning abilities, ranging from generating coherent narratives to solving complex mathematical problems and even coding~\cite{lag+22,rgg+23,llzm24}. These success stories have established Transformers as the de facto framework for advancing NLP research and applications, making the study of their effectiveness a central topic in modern NLP.

Despite the notable success of Transformer architectures, concerns regarding their inherent limitations are rising. For instance, several studies~\cite{lpk+22,xlll24,hym+25} have pointed out the issue of hallucinations, where language models produce outputs that deviate from their training data and user prompts, often leading to unreliable or misleading content. This challenge has motivated researchers to examine the theoretical foundations of Transformers. In particular, recent work~\cite{lag+22,llzm24,cll+24,cpw24} has used complexity theory to model the forward computation of Transformers as logic circuits, which allows for establishing bounds on their expressive power. These analyses suggest that Transformers belong to a relatively weak complexity class with different levels of uniformity, namely $\mathsf{TC}^0$, meaning they may struggle to solve certain problems, such as parity problems~\cite{v84} and arithmetic formula evaluation~\cite{f93}, that are known to be $\mathsf{NC}^1$-hard unless $\mathsf{TC}^0 = \mathsf{NC}^1$ holds.  

Recently, a beautiful work~\cite{ks24} provides an inspiring theoretical analysis showing that Transformers can not learn the uniform $k$-parity problem, where parity checking belongs to the $\mathsf{NC}^1$ complexity class~\cite{pip79}. However, whether Transformers can provably represent simpler logic functions within the $\mathsf{TC}^0$ class remains an open question. Previous complexity-theoretic results indicate that Transformers can represent such simple logic functions (e.g., $\mathsf{AND}$, $\mathsf{OR}$, and majority gates) because Transformers themselves belong to the $\mathsf{TC}^0$ class, but these results only address the network’s expressiveness when its parameters are chosen arbitrarily. In practice, Transformers are trained using gradient descent-based methods (e.g., SGD~\cite{rm51} or Adam~\cite{kb14}), which constrain the evolution of their parameters during optimization. This observation motivates our key research question:

\begin{center}
    {\it Can Transformers learn simple logic functions in $\mathsf{TC}^0$ using gradient descent?}
\end{center}

To answer this, we move beyond analyzing forward computation alone and examine the training dynamics of Transformers. Although a Transformer with ideal parameter settings can compute simple Boolean functions, the constraints imposed by gradient descent may prevent the model from reaching these optimal settings. In our work, we focus on a simplified variant of the Transformer architecture (details provided in Section~\ref{sec:prelim:tf_model}) and consider a finite training set of $d$ samples, where each sample is a binary string paired with the output of a basic majority function. We prove that, under both $n=\mathrm{poly}(d)$ and $n=\exp(\Omega(d))$ number of training samples settings, even after $\poly(d)$ queries to the gradient oracle, the generalization error, i.e., the gap between the Transformer’s output and the true Boolean function, still remains substantially large, with this error gap growing exponentially with $d$.

Our main contribution is that we establish a rigorous theoretical framework that provides generalization error lower bounds for any differentiable parametric model, including the transformer, learning the majority function under polynomial (Theorem~\ref{thm:finite_sample_majority_2}) or exponential (Theorem~\ref{thm:finite_sample_majority_1}) sample settings. 
Our analysis introduces novel combinatorial and probabilistic tools, which enable precise characterizations of gradient variance and generalization error.

\paragraph{Roadmap.} We present the relevant works in Section~\ref{sec:related_work}. In Section~\ref{sec:prelim}, we introduce some important basic facts and concepts. In Section~\ref{sec:main}, we introduce the main findings about the hardness of learning the majority function with Transformers. We conclude in Section~\ref{sec:conclusion}.
\section{Related Works}\label{sec:related_work}

\paragraph{Computation Limits of Transformers.}
Transformers have demonstrated remarkable performance in natural language processing but still exhibit notable limitations in mathematical computation tasks~\citep{c22}. To better understand these constraints, recent research has focused on characterizing their computational capabilities through the analysis of two primary Transformer variants: (1) average-head attention Transformers, which convert the attention probability vector into a one-hot distribution by assigning a probability of 1 to the maximum entry and 0 to others; and (2) softmax-attention Transformers, where the attention probabilities are computed via a softmax function, characterized by $\mathsf{Softmax}(X) = \diag(\exp(X) \cdot {\bf 1}_n)^{-1} \cdot \exp(X)$. An interesting work~\cite{mss22} demonstrated that the average-head attention Transformer can recognize languages beyond the complexity class $\mathsf{AC}^0$, yet remains within the capability of constant-depth threshold circuits, placing them in the non-uniform $\mathsf{TC}^0$ class.~\cite{lag+22} established that softmax-attention Transformers are similarly contained within non-uniform $\mathsf{TC}^0$. Extending this line of inquiry,~\cite{ms23} introduced a generalized similarity function framework applicable to various similarity mappings, demonstrating that softmax-attention Transformers fall within $\mathsf{L}$-uniform $\mathsf{TC}^0$. By translating Transformer operations into first-order logic with majority quantifiers ($\mathsf{FOM}$),~\cite{ms23_neurips,imm98} showed that $\mathsf{Dlogtime}$-uniform $\mathsf{TC}^0$ circuits can approximate Transformers with softmax-attention. Building on this foundation,~\cite{chi24} enhanced accuracy and precision, eliminating approximation errors for average-head attention Transformers and improving the precision of softmax-attention Transformers from $O(\log n)$ to $\poly(n)$. They further demonstrated that softmax-attention Transformers, even with an absolute error bounded by $2^{-O(\poly(n))}$, remain in the $\mathsf{Dlogtime}$-uniform $\mathsf{TC}^0$ clas. Despite the effectiveness of complexity-theoretic analyses in establishing Transformer expressivity, these approaches assume an arbitrary range of parameter values and may overlook the training dynamics imposed by gradient descent. In contrast, our work examines the hardness of the majority problem, which is solvable within the $\mathsf{TC}^0$ class, but may still remain challenging to learn when gradient descent is taken into account. And here are more related works aiming at exploring the capability of transformer such as~\cite{ks24,losw24,sfh+24,hhc+24,mfb+24,kls+25,ossw24,cll+25_icl,cls+25,tly+25}.

\paragraph{Complexity and Neural Networks.}
Circuit complexity, a fundamental field within theories of computing, investigates specific capabilities and limitations of computational models built from logical gates. By categorizing Boolean circuits into distinct complexity classes, researchers can analyze machine learning models and uncover the inherent limits to their computational efficiency. An essential finding with significant implications for machine learning is the hierarchy among circuit complexity classes, specifically the inclusion relationship $\mathsf{AC}^0 \subset \mathsf{TC}^0 \subseteq \mathsf{NC}^1$. However, determining whether the classes $\mathsf{TC}^0$ and $\mathsf{NC}^1$ coincide is still unresolved~\cite{v99,ab09}. Analyzing neural network structures through the lens of circuit complexity has emerged as an effective approach for assessing their computational strengths and limitations. Recently, Transformers and two main Transformer variants, Average-Head Attention Transformers (AHATs) and SoftMax-Attention Transformers (SMATs), have attracted significant attention. Existing studies illustrate that AHAT models can be efficiently approximated by non-uniform circuits with threshold operations and constrained to constant depth (thus within the $\mathsf{TC}^0$ complexity class)~\cite{mss22}. Complementary research indicates that SMAT models exhibit similar efficiency under L-uniform circuit simulations~\cite{lag+22}. Circuit complexity frameworks have successfully analyzed neural network architectures beyond standard Transformers, including Transformers with Rotary Position Embedding (RoPE)~\cite{sal+24}, which are broadly applied in large language models~\cite{cll+24}. The Mamba architecture~\cite{gd23}, a newer model gaining attention, has similarly been classified into the $\mathsf{DLOGTIME}$-uniform $\mathsf{TC}^0$ family~\cite{cll+24_mamba}. Furthermore, Hopfield networks, originally developed as associative memory models, have recently been demonstrated to possess complexity bounds consistent with $\mathsf{TC}^0$ circuits~\cite{lll+24_hopfield}. However, despite these advancements, applying circuit complexity approaches to analyze Graph Neural Networks (GNNs) remains relatively unexplored. Although previous studies~\cite{bkm+20,g24_gnn,lls+25_gnn} have addressed the computational capabilities of GNNs using circuit models, their perspectives differ significantly from our work. 

\section{Preliminary}\label{sec:prelim}
In Section~\ref{sec:prelim:basic_fact}, we collect simple but useful results on binomial coefficients and their expansions. In Section~\ref{sec:prelim:coefficient}, we define and apply a coefficient-extraction operator for polynomials. In Section~\ref{sec:prelim:cal_bio}, we further derive advanced identities and asymptotic estimates for binomial expressions. In Section~\ref{sec:prelim:majority} we formalize the setup of learning a hidden $k$-majority function and explain how it differs from parity. In Section~\ref{sec:prelim:loss_majority}, we introduce both population and empirical losses to quantify the learning objective. In Section~\ref{sec:prelim:tf_model}, we detail the token embedding and parameter setup for our specialized transformer architecture.

\paragraph{Notations.} Considering an arbitrary positive integer $m$, we let $[m]:=\{1,2,\ldots,m\}$. For a $d$ dimensional vector $x\in\R^d$, we use $x_j$ to denote its $j$-th entry. We use $\circ$ to denote the composition of functions. In this paper, $x$ denotes empirical samples and input binary vectors (see Sections~\ref{sec:prelim:majority} and \ref{sec:prelim:loss_majority}), while $\mathsf{x}$ represents transformer token vectors (see Section~\ref{sec:prelim:tf_model}). 
We use $\mathbf{1}[\cdot]$ to denote the indicator function, which outputs $1$ when the condition inside the square brackets is true and outputs $0$ otherwise.

\subsection{Basic Facts}\label{sec:prelim:basic_fact}
 
We state a simple fact for binomial coefficient.
\begin{fact}[Folklore, See example in textbook \cite{clrs22}
] 
\label{fac:cardinality}
Let $d$ and $k$ denote positive integers.
    Let $d \geq k \geq 2$. 
    We have 
    \begin{align*}
    (d/k)^k \leq \binom{d}{k} \leq (ed/k)^k.
    \end{align*}
\end{fact}

\begin{fact}[Informal version of Fact~\ref{fac:binomial_thm_append} in Section~\ref{sec:append_prelim:basic_fact}]\label{fac:binomial_thm}
    Considering $t,n,k \in \mathbb{Z}$, we have
    \begin{align*}
        (1+t)^n 
        = & ~ \sum_{k=0}^n \binom{n}{k} t^k.
    \end{align*}
\end{fact}

\begin{lemma}[Clipping Property, Informal version of Lemma~\ref{lem:clipping_append} in Section~\ref{sec:append_prelim:basic_fact}]\label{lem:clipping}

If the following conditions hold:
\begin{itemize}
    \item We denote $f$ as a function such that $f: \{\pm 1\}^d \to \R$.
    \item Let the clipped version of $f$ be $\ov{f}(x) = \min\{\max\{f(x), -1\}, 1\}$.
    \item We denote $g$ as any binary function which satisfies $g: \{\pm 1\}^d \to \pm1$.
\end{itemize}
Then the following statement is true:
\begin{align*}
    (g(x) - f(x))^2 \ge (g(x)-\ov{f}(x))^2.
\end{align*}
\end{lemma}

\subsection{Coefficient Operator}\label{sec:prelim:coefficient}
In this section, we define and then introduce the coefficient operator related to our results.
\begin{definition}\label{def:k_coefficient_operator}
   Let $x\in\R$ be an arbitrary real number. Let $f(t) := a_0 + a_1 t + a_2 t^2 +\dots +a_k t^k$ be an order-$k$ polynomial function with variable $t$. Let $i\in\{0,1,\ldots,k\}$. We use operator $\mathsf{Coeff}_i[f(t)]$ to extract the coefficient of $t^i$, i.e., 
   \begin{align*}
       \mathsf{Coeff}_i[f(t)]  := a_i.
   \end{align*}
\end{definition}

\begin{fact}[Relationship to Combinations, Informal version of Fact~\ref{fac:binom_to_tk_append} in Section~\ref{sec:append_prelim:coefficient}]\label{fac:binom_to_tk}
If the following conditions hold:
\begin{itemize}
    \item Let $n$ and $k$ be non-negative integers and $n \geq k$. 
    \item  Let $\mathsf{Coeff}_i[f(t)]$ be Definition~\ref{def:k_coefficient_operator}. 
    \item Let $t\in\R$. 
\end{itemize} Then we can show
    \begin{align*}
        \binom{n}{k} = \mathsf{Coeff}_k[ (1+t)^n ].
    \end{align*}
\end{fact}

\begin{fact}[Index Shift]\label{fac:coeff_op_change_index}
    If the following conditions hold:
    \begin{itemize}
        \item Let $k\geq j \geq 0$ be non-negative integers. 
        \item Let $\mathsf{Coeff}_k[\cdot]$ be the coefficient operator as defined in Definition~\ref{def:k_coefficient_operator}.
    \end{itemize}
    The following statement is true:
    \begin{align*}
        \mathsf{Coeff}_{k-j}[f(t)] = \mathsf{Coeff}_{k}[f(t)t^j].
    \end{align*}
\end{fact}
\begin{proof}
Multiplying the polynomial function $f(t)$ by $t^j$ shifts every term's degree by $j$, so the coefficient of $t^k$ in $f(t)t^j$ is exactly the coefficient of $t^{k-j}$ in $f(t)$. Thus, we complete the proof.
\end{proof}

\begin{fact}[Linearity]\label{fac:coeff_op_lin}
    If the following conditions hold:
    \begin{itemize}
        \item Let $k\geq j \geq 0$ be non-negative integers . 
        \item Let $\mathsf{Coeff}_k[\cdot]$ be the coefficient operator as defined in Definition~\ref{def:k_coefficient_operator}.
        \item Let $\alpha_1,\alpha_2,\ldots,\alpha_m\in\R$ be real numbers.
        \item Let $f_1(t), f_2(t), \ldots, f_m(t)$ be polynomial functions with variable $t$. 
    \end{itemize}
    The following statement is true:
    \begin{align*}
        \mathsf{Coeff}_{k}[\sum_{i=1}^m\alpha_if_i(t)] = \sum_{i=1}^m\alpha_i\mathsf{Coeff}_{k}[f_i(t)].
    \end{align*}
\end{fact}
\begin{proof}
    Since the coefficient operator extracts the coefficient corresponding to a specific power of $t$ and polynomial addition and scalar multiplication operate term-by-term, this linearity trivially holds.
\end{proof}

\subsection{Calculations about Binomial Coefficients}\label{sec:prelim:cal_bio}
We introduce the important facts about calculations in binomial coefficients.
\begin{lemma}[Informal version of Lemma~\ref{lem:append_ratio_of_card_pplus_p} in Section~\ref{sec:append:binom_coeff}
]\label{lem:ratio_of_card_pplus_p}
Let $m,d,k$ be non-negative integers. Assume $d\geq m\geq 0.5k$ and $d-m\ge 0.5k$.  Let $A:= \sum_{j=0}^{k/2}\binom{m}{j}\binom{d-m}{k-j}$.   Let $B:= \frac{1}{2} \binom{d}{k}$.    Let $\Gamma(t) := (1+t)^{d-m}$.  Let $\Delta(t):= \sum_{j=0}^{k/2} \binom{m}{j}t^j -\sum_{j=k/2+1}^m \binom{m}{j} t^j$.  
Then we can show
\begin{itemize}
\item Part 1. 
$
A = \frac{1}{2} \binom{d}{k} +\frac{1}{2} \mathsf{Coeff}_{k}[\Gamma(t) \Delta(t)].
$
\item Part 2.
$
    | \frac{A}{B} - 1 | \leq e^{-\Omega(d)}.
    $
\end{itemize}
\end{lemma}

\subsection{The Majority Problem}\label{sec:prelim:majority}
In this section, we introduce some definitions about the majority problem.
\begin{definition}[The Majority Function]\label{def:maj_func}
    Let $d\in\mathbb{N}_+$ and $S\subseteq[d]$ be the support of a $d$-dimensional vector. Let $x\in\{\pm1\}^d$ be a $d$-dimensional binary vector. The output of the majority function is defined as:
    \begin{align*}
        \mathsf{MAJ}(x,S):= \begin{cases}
            +1, & 
        \sum_{j\in S}x_j\geq 0; \\
        -1, & 
        \sum_{j\in S}x_j < 0.
        \end{cases}
    \end{align*}
\end{definition}

\begin{definition}[The $k$-Majority Problem]\label{def:maj_prob}
    Let $d \ge k \ge 2$ be positive integers. Let $\S$ denote the set of all $k$-element subsets of $[d]$. Let there be a fixed but unknown support $S \in \S$, which is sampled uniformly at random from the set $\S$. The $k$-majority problem aims to find the underlying $p$ by querying a finite number of $d$-bit inputs: 
    \begin{align*}
        x := (x_j)_{j=1}^d \sim \mathrm{Unif}(\{\pm1\}^d) \in \R^d,
    \end{align*}
    and their corresponding labels $y := \mathsf{MAJ}(x, S)\in\{\pm 1\}$.
\end{definition}

\begin{remark}[Difference between Parity and Majority]\label{rmk:diff_par_maj}
    Let $x_1:=(1,1,\underbrace{-1,\ldots,-1}_{d-2})^\top \in \R^d$ and $x_2 :=(-1,-1,\underbrace{1,\ldots,1}_{d-2})^\top \in \R^d$. Let the support of input vectors be $S = [d]$. We have $\mathsf{PAR}(x_1,S)\neq \mathsf{PAR}(x_2,S)$, while $\mathsf{MAJ}(x_1,S)\neq \mathsf{MAJ}(x_2,S)$. This indicates the inherent difference between the parity problem and the majority problem. 
\end{remark}

\subsection{Loss Function for the Majority Problem}\label{sec:prelim:loss_majority}
In this section, we introduce some definitions about loss functions.

\begin{definition}[Population Loss]\label{def:pop_loss}
    Let $\{f_\theta ~|~ \theta \in \Theta\}$ be any differentiable (with respect to $\theta$) parametric model and let $y$, we define the population loss
    \begin{align*}
        \wt{L}_S(\theta) := \E_{x \sim \mathrm{Unif}(\{\pm1\}^d) } [ (y - f_\theta(x))^2 ],
    \end{align*}
    where $y = \mathsf{MAJ}(x,S) \in \R$ as defined in Definition~\ref{def:maj_prob}. 
\end{definition}

\begin{definition}[Empirical Loss]\label{def:empirical_loss}
  Let $f_{\theta} : \{\pm1\}^d \to \R$ be a differentiable model parameterized with $\theta$.  
  Let there be $n$ i.i.d. empirical samples $\{(x_i, y_i)\}_{i=1}^n$ such that $x_i \sim \mathrm{Unif}(\{\pm1\}^d)$ and $y_i = \mathsf{MAJ}(x_i,S)$. 
  We define the empirical loss as
    \begin{align*}
    L_{n,S}(\theta) 
     :=  \frac{1}{2n} \sum_{i=1}^n (y_i - f_{\theta}(x_i))^2. 
    \end{align*}
    For notation simplicity, we denote the empirical inner product of two functions $f:\R^d\rightarrow \R$ and $g:\R^d\rightarrow \R$ by $\langle f, g \rangle_n := \frac{1}{n} \sum_{i=1}^n f(x_i)g(x_i)$ and the corresponding empirical norm as $\|f\|^2_n := \langle f, f \rangle_n$. Thus, the empirical loss can also be written as: 
    \begin{align*}
        L_{n,S}(\theta) = \frac{1}{2} \|\mathsf{MAJ}(\cdot,S) - f_{\theta}\|_n^2.
    \end{align*}
\end{definition}

\subsection{Transformer Model}\label{sec:prelim:tf_model}
In this subsection, we define the transformer model.
\begin{definition}\label{def:ell_L}
For notation convenience, we define $\ell: = d +k -1$ and define $L:= n + \ell$.
\end{definition}
Here, we introduce the definition of token embedding.

\begin{definition}[Token Embedding]\label{def:token_embed}
    Let $n, d, k \in \mathbb{N}_+$ where $n$ is the number of samples, $d$ is the input dimension, and $k$ is the number of majority bits.
    Given $n$ input samples $\{  x_1, x_2, \cdots, x_n \} \subset \{-1, +1 \}^d$. Recall that $\ell = d+k-1$. 

    We define $\x_j$ in two cases
    \begin{itemize}
    \item
    
    for each $j \in [d]$, we define the token vector $\x_j \in \R^n$ as $ 
        \x_j := \begin{bmatrix} x_{1,j} & x_{2,j} & \cdots & x_{n,j} \end{bmatrix}^\top  
    $
    where $x_{i,j}$ is the $j$-th bit of empirical sample $x_i$,
    
    \item  $\forall j \in \{d+1, \cdots, \ell \}$.  we define the dummy tokens as $\x_j := {\bf 0}_n$.
    \end{itemize}
    
    Each token $\x_j$ is concatenated with a one-hot positional encoding $e_j \in \R^{
    \ell}$ to form. For each $j \in [\ell]$, we define $p_j \in \R^{n + \ell}$ as follows
    $
        p_j := \begin{bmatrix} \x_j \\ e_j \end{bmatrix} 
    $.
    
\end{definition}

We introduce the attention score below.
\begin{definition}[Attention Score]\label{def:attn}
Recall that $\ell$ and $L$ is Defined as Definition~\ref{def:ell_L}. 
     Define $W \in \R^{\ell \times \ell }$ as our parameter matrix. The Query, Key, and Value matrices are defined as size $Q, K, V \in \R^{n \times L}$ . 
    In our specific model, we formulate the key-query product $K^\top Q$ and value matrix $V$ as follows:
    \begin{align*}
        K^\top Q :=  ~ 
        \begin{bmatrix}
        0_{n\times n} & 0_{n\times \ell} \\
        0_{\ell \times n} & W
        \end{bmatrix} \in \R^{ L \times L } , ~~~
        V := ~
        \begin{bmatrix}
        I_{n\times n} & 0_{n\times \ell}
        \end{bmatrix} \in \R^{n \times L }.
    \end{align*}
    This parameterization ensures that attention scores depend only on positional encodings, with the key-query product containing our parameters $W$ in the lower right block.
\end{definition}

\section{Hardness of Majority Problem}\label{sec:main}
In Section~\ref{sec:main:main}, we state two key hardness theorems establishing why finite-sample majority learning is difficult. In particular, we consider both the exponential sample complexity case and the polynomial sample complexity case. In Section~\ref{sec:main:gradient_variance}, we define gradient variance for majority learning and derive tight upper bounds. In Section~\ref{sec:main:def_oracle} we introduce an approximate gradient oracle that masks the choice of the underlying majority bits. In Section~\ref{sec:main:oracle} we prove the guarantees showing that the oracle can conceal the true support with high probability. In Section~\ref{sec:main:l_infty_loss} we show a lower bound for the worst-case error. We show a similar lower bound on the MSE in Section~\ref{sec:main:mse}.

\subsection{Main Theorem}\label{sec:main:main}

Here, we introduce our main theorem.

\begin{theorem}[Main Result I, Hardness of Polynomial-Sample Majority, Informal version of Theorem~\ref{thm:finite_sample_majority_2_append} in Section~\ref{sec:append_main:main}]\label{thm:finite_sample_majority_2}
Let $c = 4c_1 + 4c_2 + 2c_3 + 2c_4 + 1$ where $c_i > 0$ for all $i \in [4]$.  $n = \Omega( d^{c} ) $.
  Let $k = \Theta(d)$.
  $\|\nabla f_{\theta}\|_2 =  O (d^{c_1} )$.
  Let $S \in \S$ be a fixed underlying support.
  Let $T= O(d^{c_3})$. 
Then the following statement is true:
\begin{itemize}
    \item There exists an $O (d^{-c_2} )$-approximate gradient oracle $\wt{\nabla}$
  such that with probability $1 - e^{-\Omega(d)}$ over the random choice of samples, 
  for \emph{any} (possibly randomized) iterative algorithm ${\cal A}$ making at most $T$ queries 
  to $\wt{\nabla}L_n$, the output $\theta({\cal A})$ satisfies
    \begin{align*}
       \E_{p,x} [ (\mathsf{MAJ}(x,S) - f_{\theta({\cal A})}(x) )^2 ]
     \geq  1 - O (d^{-c_4} ).
    \end{align*} 
\end{itemize}
\end{theorem}

\begin{theorem}[Main Result II, Hardness of Exponential-Sample Majority, Informal version of Theorem~\ref{thm:finite_sample_majority_1_append} in Section~\ref{sec:append_main:main}]\label{thm:finite_sample_majority_1} 
Let $n = e^{\Omega(d)}$. Let $k = \Theta(d)$.
 Let $\|\nabla f_{\theta}\|_2 =  \poly(d)$. 
  Let $S \in \S$ be a fixed underlying support. Let $T = \poly(d)$. 
 
Then the following statement is true:
\begin{itemize}
    \item There exists an $e^{-\Omega(d)}$-approximate gradient oracle $\wt{\nabla}$ 
  such that with probability $1 - e^{-\Omega(d)}$ over the random choice of samples, 
  for \emph{any} (possibly randomized) iterative algorithm ${\cal A}$ that makes at most $T$ queries 
  to $\wt{\nabla} L_n$, the output $\theta({\cal A})$ satisfies
  \begin{align*}
      \E_{p,x} [ (\mathsf{MAJ}(x,S) - f_{\theta({\cal A})}(x) )^2 ]
     \geq  1 - e^{-\Omega(d)}.
  \end{align*}
\end{itemize}
\end{theorem}

Theorem~\ref{thm:finite_sample_majority_2} shows that when training with a polynomial number of samples, any differentiable parametric model, including the transformer, cannot learn the majority function, i.e., the generalization error is large with high probability. 
Furthermore, in Theorem~\ref{thm:finite_sample_majority_1}, we can extend our results to the exponential number of training samples case with a new corresponding error bound. 

\subsection{Bounding the Gradient Variance}\label{sec:main:gradient_variance}
In this section, we introduce some results about bounding the gradient variance.

\begin{definition}[Gradient Variance]\label{def:grad_var}
    Let $\S$ be the set of all the $k$-size support sets over $\{-1,1\}^d$, and $S\in\S$ be one fixed underlying support. Let $L_{n,S}(\theta)$ be the empirical loss from Definition~\ref{def:empirical_loss} for learning majority function $\mathsf{MAJ}(\cdot, S)$ with model parameters $\theta \in \Theta$, and $n$ be the number of training samples.
    We define the gradient variance as follows:
    \begin{align*}
        \mathrm{Var}_n(\theta; \S) 
        :=  \E_{S_1\in \S}[\|\nabla L_{n,S_1}(\theta) - \E_{S_2\in \S}[\nabla L_{n,S_2}(\theta)]\|_2^2]. 
    \end{align*}
\end{definition}

\begin{fact}[Informal version of Fact~\ref{fac:grad_var_tool_append} in Section~\ref{sec:append_main:gradient_variance}]\label{fac:grad_var_tool}
    If the following conditions hold:
        \begin{itemize}
            \item Let $\S$ be the set of all the $k$-size support sets over $\{\pm 1\}^d$.
            \item Let there be $n$ fixed empirical samples $\{(x_i, y_i)\}_{i=1}^n$ as defined in Definition~\ref{def:empirical_loss}. 
            \item Let the empirical inner product $\langle \cdot, \cdot \rangle_n$ empirical norm $\| \cdot \|_n$ as defined in Definition~\ref{def:empirical_loss}. 
            \item Let $h_S:=\maj(x,S)$ denote the majority function in Definition~\ref{def:maj_func} for any $S\in\S$.
            \item Let the gram matrix be $G_{\S}:=(\langle  h_{S_1},h_{S_2}\rangle_n)_{S_1,S_2\in \S}$, in which the largest eigenvalue $\lambda_{\max}(G_{\S})$ is upper bounded. 
        \end{itemize}
    The following statement is true:
    $
        \sum_{S\in\S}\langle f, h_S\rangle_n^2 \leq \lambda_{\max}(G_{\S}) \|f\|_n^2.
    $
\end{fact}

Then we give the important bound on gradient variance for majority learning.

\begin{lemma}[Bound on Gradient Variance for Majority Learning, Formal version of Lemma~\ref{lem:grad_var_append} in Section~\ref{sec:append_main:gradient_variance}]\label{lem:grad_var}
    Let $\mathrm{Var}_n(\theta; \S)$ be the gradient variance as defined in Definition~\ref{def:grad_var}. Then, with probability at least $1 - e^{-\Omega(d)}$ over random sampling, the variance of the gradient of the empirical loss with respect to the target majority is bounded as:
    $
        \mathrm{Var}_n(\theta; \S) 
        \leq   2\sqrt{\frac{d}{n}} \sup_{\theta,x} \|\nabla f_\theta(x)\|_2^2.
    $
\end{lemma}

\begin{corollary}\label{cor:grad_var}
    Let $\mathrm{Var}_n(\theta; \S)$ be the gradient variance as defined in Definition~\ref{def:grad_var}. Then for any $\epsilon > 0$, we have: 
    \begin{align*}
        \Pr[\|\nabla L_{n,S_1}(\theta) - \E_{S_2\in \S}[\nabla L_{n,S_2}(\theta)]\|_2 > \epsilon] \leq 2\epsilon^{-2} \sqrt{\frac{d}{n}} \sup_{\theta,x} \|\nabla f_\theta(x)\|_2^2.
    \end{align*}
\end{corollary}
\begin{proof}    
    The result directly follows from the Chebyshev's inequality and Lemma~\ref{lem:grad_var}. 
\end{proof}

\subsection{Definition of Gradient Oracle}\label{sec:main:def_oracle}
In this section, we introduce some basic definitions of the gradient oracle.

\begin{definition}[Mean Gradient]\label{def:mean_grad}
    We use $\ov{\nabla}(\theta)  := \E_{S \in \S} [\nabla L_{n,S}(\theta) ]$ to denote the mean gradient for all $S\in \S$.
\end{definition}

\begin{definition}[$\epsilon$-Approximate Gradient Oracle in \cite{s18}]
\label{def:approx_grad_oracle}
We recall that $L_{n,S}(\theta)$ is our empirical loss depending on a majority function $\mathsf{MAJ}(\cdot,S), S\in \S$ and parameters $\theta$. Fix $\epsilon>0$. For any $\theta$ and any target support $S\in \S$, define 
\begin{align*}
       \wt{\nabla} L_{n,S}(\theta) := 
   \begin{cases}
      \ov{\nabla}(\theta) ,
      & \mathrm{if~}  \|\nabla L_{n,S}(\theta)  -  \ov{\nabla}(\theta) \|_2 \geq  \epsilon;\\
      \nabla L_{n,S}(\theta),
      & \mathrm{otherwise}.
   \end{cases}
\end{align*}
\end{definition}

\begin{remark}
    In other words, if the true gradient $\nabla L_{n,S}(\theta)$ deviates from the mean gradient by more than $\epsilon$, then we truncate to the mean gradient; if it is within $\epsilon$, we return the true gradient.
\end{remark}

Intuitively, this construction ensures that the oracle only reveals the true gradient when it is already close to the overall average. Whenever $\nabla L_{n,S}(\theta)$ contains ``too much'' information about the specific choice of $S$, the oracle reverts to returning the mean gradient $\ov{\nabla}(\theta)$, thus concealing the identity of $S$.

\subsection{Guarantee of the Oracle}\label{sec:main:oracle}
In this section, we introduce some important properties of the oracle.  
\begin{lemma}[Guarantee of the Oracle, Informal version of Lemma~\ref{lem:privacy_oracle_append} in Section~\ref{sec:append_main:oracle}]\label{lem:privacy_oracle}
 We represent the subset of the hypothesis space where the oracle defaults to the mean gradient by $Q\subseteq \S$.
  Let $T$ be the number of optimization steps.
 Let $d$ be the dimensionality of the input.  
  Let $n$ be the sample size.
 Let $f_\theta(x)$ be the model's prediction function.
 
Then the following statement is true: 
\begin{itemize}
    \item During $T$ optimization steps, the oracle reveals no information 
    about the true support $S\in\S$ with probability at least:
    $
\Pr[Q] \ge 1 - \frac{2T}{\epsilon^2}\sqrt{\frac{d}{n}} \sup_{\theta,x} \|\nabla f_\theta(x)\|^2.
$
\end{itemize} 
\end{lemma}

\subsection{Lower Bounding the \texorpdfstring{$L_{\infty}$}{} Loss}\label{sec:main:l_infty_loss}

Here we introduce our lower bound result on $L_{\infty}$ error in this section.

\begin{definition}[Automorphism]
Let $Q$ denote an arbitrary set. An automorphism $\sigma: Q \to Q$ is a bijection from a set $Q$ to itself that preserves the structure of $Q$. If there are no $p\in Q$ such that $\sigma(q)=q$ (i.e., $\emptyset = \{q\in Q : \sigma(q) = q\}$), we say that automorphism $\sigma$ has no fixed points.
\end{definition}

\begin{theorem}[Lower Bound on $L_\infty$ Error, Informal version of Theorem~\ref{thm:lower_l_infty_error_append} in Section~\ref{sec:append_main:l_infty_loss}]\label{thm:lower_l_infty_error}
If the following conditions hold:
\begin{itemize}
    \item Let $Q \subseteq \S$ be a noninformative set with an automorphism $\sigma: Q \to Q$ having no fixed points. 
    \item Let function $f $ be $f_{\theta({\cal A})}: \{\pm 1\}^d \to \R$. 
\end{itemize}
Then the following statement is true:
\begin{align*}
\E_{S \in \S}[\sup_{x}|\mathsf{MAJ}(x, S) - f_{\theta({\cal A})}(x)|] \geq \Pr[S \in Q].
\end{align*}
\end{theorem}

\subsection{Lower Bounding the Mean Squared Error}\label{sec:main:mse}
In this section, we introduce the lower bound result on mean squared error.
\begin{lemma}[Informal version of Lemma~\ref{lem:ppp_append} in Section~\ref{sec:append_main:mse}]\label{lem:ppp}
    If $k=\Theta(d)$, it holds with probability at least $1-e^{-\Omega(d)}$ over random sampling that
    $
    |\E_{S\in \S}[\maj(x, S)]| \le e^{-\Omega(d)}.
    $
\end{lemma}

\begin{theorem}[Lower Bound on Mean Squared Error, Informal version of Theorem~\ref{thm:mse_bound_append} in Section~\ref{sec:append_main:mse}]\label{thm:mse_bound}
If the following conditions are true:
\begin{itemize}
    \item We have a function $f_{\theta({\cal A})}: \{\pm 1\}^d \to \R$ learned by algorithm ${\cal A}$ using $n$ samples. 
    \item Let $Q \subseteq \S$ be a noninformative set.
    \item We denote the number of iterations in optimization by $T$.
    \item Let $\epsilon > 0$ denote an arbitrary gap between a single gradient and the expectation of gradient as defined in Corollary~\ref{cor:grad_var}. 
\end{itemize}
Then the following statement is true:
\begin{align*}
\E_{S \in \S, x}[(\mathsf{MAJ}(x,S) - f_{\theta({\cal A})}(x))^2] \geq 1 - \frac{4T}{\epsilon^2}\sqrt{\frac{d}{n}} \sup_{\theta,x} \|\nabla f_\theta(x)\|^2 - e^{-\Omega(d)}.
\end{align*}
\end{theorem}

\section{Discussion}

In this section, we discuss the difference between our work and a previous work~\cite{ks24}. Both works share a similar high-level analysis framework, while the major difference lies in the type of binomial coefficients being bounded. Specifically,~\cite{ks24} needs to bound a series of binomial coefficients related to the parity function, while we analyze the the series of binomial coefficients related to the majority function.

To analyze the binomial coefficients, both works use the operator $\mathsf{Coeff}_k[f(t)]$, which denotes the $k$-th coefficients with respect to variable $t$ in polynomial function $f(t)$, as shown in Definition~\ref{def:k_coefficient_operator}. 

\paragraph{Previous Approach for Parity Binomial Coefficients~\cite{ks24}.} 
Let $A_{\mathrm{par}}:= \sum_{j=0}^{m/2} {m \choose 2j} { {d-m} \choose {k-2j} }$. The intuition behind $A_{\mathrm{par}}$ is that it represents the cardinality of the set $\{S\in\mathcal{S}:\mathsf{PAR}(x, S)=1\}$, which denotes the number of support sets for which the parity function outputs $1$ on a fixed $x$. Thus, let the number of $-1$s in $x$ be $m$. To ensure that the parity function outputs $1$, we must select an even number of $-1$s from $x$ to include in the support set. For each such choice, we can then select the remaining elements arbitrarily from the positions where $x_i = +1$. This directly yields the definition of $A_{\mathrm{par}}$, which sums over all even integers between $0$ and $m$, computing the number of valid support sets for each such choice. Let $B = \frac{1}{2} {d \choose k}$. The goal is to show that (see Lemma 8 of~\cite{ks24}, page 17):
\begin{align*}
|\frac{A_{\mathrm{par}}}{B}- 1| \leq e^{-\Omega(d)}.
\end{align*}

Let $\Gamma(t) = (1+t)^{d-m}$. Then, they try to rewrite the equation as 
\begin{align*}
  A_{\mathrm{par}} = \frac{1}{2} {d \choose k}  + \frac{1}{2} \mathsf{Coeff}_k[ \Gamma(t) \Delta_{\mathrm{par}}(t) ].
\end{align*}

The key question in their case is determining $\Delta_{\mathrm{par}}(t)$, and they show that $\Delta_{\mathrm{par}}(t) = (1 - t)^m$. This reduces the original problem to proving the following bound:
\begin{align*}
| \frac{ \mathsf{Coeff}_k[ \Gamma(t) \Delta_{\mathrm{par}}(t) ] }{ {d \choose k} } | \leq e^{-\Omega(d)}.
\end{align*}

\paragraph{Our Approach for Majority Binomial Coefficients.}

Let $A_{\mathrm{maj}}:= \sum_{j=0}^{k/2}\binom{m}{j}\binom{d-m}{k-j}$. Our intuition for $A_{\mathrm{maj}}$ differs from the parity case, where $A_{\mathrm{maj}}$ represents the cardinality of the set $\{S \in \mathcal{S} : \mathsf{MAJ}(x, S) = 1\}$, i.e., the number of support sets for which the majority function outputs $1$ on a fixed $x$. We let the number of $-1$s in $x$ be $m$, and we focus on the non-trivial case where $d - m \geq 0.5k$ and $m \geq 0.5k$. This ensures that the output of $\mathsf{MAJ}(x, S)$ depends meaningfully on the choice of support set $S$, rather than being trivially determined due to an overwhelming majority of either $-1$s or $+1$s in $x$. In this setting, we count the number of support sets $S$ such that fewer $-1$s than $+1$s are covered. Specifically, for each $j \in \{0, 1, \cdots, k/2\}$, we choose $j$ elements from the $m$ positions where $x_i = -1$, and $k - j$ elements from the $d - m$ positions where $x_i = +1$. This directly yields the definition of $A_{\mathrm{maj}}$.
Let $B = \frac{1}{2} {d \choose k}$. Our major goal is to prove that
\begin{align*}
|\frac{A_{\mathrm{maj}}}{B}- 1| \leq e^{-\Omega(d)}.
\end{align*}

 The key idea is to rewrite $A_{\mathrm{maj}}$ as follows (see Lemma~\ref{lem:ratio_of_card_pplus_p} in Section~\ref{sec:prelim:cal_bio}):
 \begin{align*}
    A_{\mathrm{maj}} = 
     & ~ \frac{1}{2} \binom{d}{k} +\frac{1}{2} \mathsf{Coeff}_{k}[ \Gamma(t) \Delta_{\mathrm{maj}} (t) ].
\end{align*}

In our case, the key question is determining $\Delta_{\mathrm{maj}}(t)$, and we show that
\begin{align*}
\Delta_{\mathrm{maj}}(t) = \sum_{j=0}^{k/2} {m \choose j} t^j -\sum_{j=k/2+1}^m {m \choose j} t^j.
\end{align*}

This reduces the original problem to proving the following bound:
\begin{align*}
| \frac{ \mathsf{Coeff}_k[ \Gamma(t) \Delta_{\mathrm{maj}}(t) ] }{ {d \choose k} } | \leq e^{-\Omega(d)}.
\end{align*}

Since the form of $\Delta_{\mathrm{maj}}(t)$ differs significantly from the form of $\Delta_{\mathrm{par}}(t)$, the bounding technique is also different. For details on how to obtain this new bound w.r.t. $\Delta_{\mathrm{maj}}(t)$, we refer the reader to Lemma~\ref{lem:append:negligible_terms} in Appendix~\ref{sec:bound_binomial_coefficient}.

\paragraph{Other Parts of the Framework.} 
In this work, we study the majority problem (see Definition~\ref{def:maj_prob}), which differentiates from the original work~\cite{ks24} by substituting the parity function $\mathsf{PAR}(x,S):=\prod_{j\in S} x_j$ with majority function in Definition~\ref{def:maj_func}. We further show the distinction between these two problems in Remark~\ref{rmk:diff_par_maj}. Besides, this difference in the studied problem also leads to slight variations in the computation of the Gram matrix, which is useful for bounding the gradient variance. We refer readers to Fact~\ref{fac:grad_var_tool} in Section~\ref{sec:main:gradient_variance} for a to check these differences.

\section{Conclusion}\label{sec:conclusion}
Our paper provides a rigorous theoretical analysis of the majority function learning problem under polynomial and exponential-sample regimes, offering explicit hardness results. By integrating combinatorial and probabilistic techniques, we not only characterize tight gradient variance bounds but also illuminate the inherent difficulties in achieving high accuracy. Our findings lay a solid foundation for future research on extending majority learning to a wider range of models and scenarios.

\ifdefined\isarxiv
\bibliographystyle{alpha}
\bibliography{ref}
\else
\bibliography{ref}
\bibliographystyle{colm2025_conference}
\fi

\newpage
\onecolumn
\appendix
\begin{center}
    \textbf{\LARGE Appendix}
\end{center}
\paragraph{Roadmap.} In Section~\ref{sec:append_def_basic} we review classical concentration bounds such as Chernoff’s and Hoeffding’s inequalities. Next, we supplement some missing proofs in Section~\ref{sec:prelim} of main text in appendix Section~\ref{sec:append_prelim}. In Section~\ref{sec:append_main}, we introduce other missing proofs in Section~\ref{sec:main} of main text. In Section~\ref{sec:bound_binomial_coefficient}, we provide a technical lemma to bound the Binomial coefficients.

\section{Definitions, Facts and Probability Tools}\label{sec:append_def_basic}
In Section~\ref{sec:append_stand_probability}, we review some standard probability tools and basic concepts. In Section
~\ref{sec:prelim:transformer_forward}, we present the formulation of forward passes in transformer.
\subsection{Standard Probability Tools and Facts}\label{sec:append_stand_probability}
In this section, we introduce some important standard probability tools.
\begin{lemma}[Chernoff inequality from~\citet{C52}] \label{lem:chernoff}
    Let $X_1, \cdots, X_n$ be independent random variables such that $X_i=1$ has probability $p_i$ and $X_i=0$ has probability $1-p_i$. We define $X = \sum_{i=1}^n X_i$ and $\mu = \E[X] = \sum_{i=1}^n p_i$. Then, the following statements are true:
    \begin{enumerate}
    \item For all $\delta>0$, we have $\Pr[X \geq \mu(1 + \delta)] \leq \exp(-\delta^2\mu/3)$ ;
    \item For all $0<\delta<1$, we have $\Pr[X \leq \mu(1 - \delta)] \leq \exp(-\delta^2\mu/2)$.
    \end{enumerate}
\end{lemma}

\begin{lemma}[Hoeffding inequality from~\citet{h94}] \label{lem:hoeffding}
    Let $X_1, \cdots, X_n$ be $n$ bounded independent random variables, where $X_i \in [a_i,b_i]$. Let $X = \sum_{i=1}^n X_i$. 
    Then, we have
    \begin{align*}
    \Pr[|X - \E[X]| \geq t] \leq 2\exp(-\frac{2t^2}{\sum_{i=1}^n (b_i - a_i)^2}).
    \end{align*}
\end{lemma}

\subsection{Transformer Forward Pass}\label{sec:prelim:transformer_forward}
In this section, we introduce the forward pass in the transformer.
\begin{definition}[Link Function]\label{def:link_func}
    We define a link function $\phi: [-1,1] \to [-1, 1]$, which we choose $\phi$ as follows $\phi(0) = -1,~~\phi(\pm 1) = 1,~~\phi'(0) = \phi'(\pm 1) = 0$, where $\phi'$ is the derivative of $\phi$. We assume $\phi$ is symmetric and sufficiently regular.
\end{definition}

Then we give the complete definition of the transformer forward pass.
\begin{definition}[Transformer Forward Pass]\label{def:transformer}
    Let $n, d, k \in \R$ be positive integers where $n$ is the number of samples, $d$ is the input dimension, and $k$ is the number of majority bits. Let $[d] = \{1,2,\hdots,d\}$ denote the set of the first $d$ positive integers. Let $x_1, \hdots, x_{\ell} \in \R^n$ be token vectors and $w \in \R^{\ell \times \ell}$ be a weight matrix. Let $\phi: [-1, 1] \to [-1, 1]$ be a link function as in Definition~\ref{def:link_func}. We define the transformer function $\mathsf{TF}: \R^{n \times \ell} \times \R^{\ell \times \ell} \to \R^{n \times \ell}$ as:  
    \begin{align*}
        \mathsf{TF}(x_1, \hdots, x_{\ell}; w) := (\wh{x}_1, \hdots, \wh{x}_{\ell}),
    \end{align*}
    where $\wh{x}_j = x_j$ for $j \in [d]$ and for $m \in \{d+1, \hdots, \ell\}$, we define:
    \begin{align*}
        \wh{x}_m :=  ~ \phi(\wh{z}_m), ~~~
        \wh{z}_m :=  ~\sum_{j=1}^{m-1} \sigma_j(w_m)x_j.
    \end{align*}
    The function $\sigma_j: \R^{\ell} \to [0,1]$ computes the attention scores as:
    \begin{align*}
        \sigma_j(w_m) := \frac{e^{w_{j,m}}}{\sum_{\alpha=1}^{m-1} e^{w_{\alpha,m}}},
    \end{align*}
    where $w_{j,m}$ is the $(j,m)$-th element of $w$, and the scores are subject to causal masking where $w_{j,m} = -\infty$ for $j \geq m$ or $m \leq d$.
\end{definition}

\section{Missing Proofs in Section~\ref{sec:prelim}}\label{sec:append_prelim}
In this section, we introduce some missing proofs in Section~\ref{sec:prelim}.
\subsection{Basic Facts}\label{sec:append_prelim:basic_fact}

We state a simple fact for the binomial coefficient.

\begin{fact}[Formal version of Fact~\ref{fac:binomial_thm} in Section~\ref{sec:prelim:basic_fact}]\label{fac:binomial_thm_append}
    Considering $t,n,k \in \mathbb{Z}$, we have
    \begin{align*}
        (1+t)^n 
        = & ~ \sum_{k=0}^n \binom{n}{k} t^k.
    \end{align*}
\end{fact}
\begin{proof}
    This directly follows from the binomial theorem $(x+y)^n=\sum_{k=0}^n \binom{n}{k} x^{n-k}y^k$ by setting $x=1, y=t$. 
\end{proof}

\begin{lemma}[Clipping Property, Formal version of Lemma~\ref{lem:clipping} in Section~\ref{sec:prelim:basic_fact}]\label{lem:clipping_append}

If the following conditions hold:
\begin{itemize}
    \item We denote $f$ as a function such that $f: \{\pm 1\}^d \to \R$.
    \item Let the clipped version of $f$ be $\ov{f}(x) = \min\{\max\{f(x), -1\}, 1\}$.
    \item We denote $g$ as any binary function which satisfies $g: \{\pm 1\}^d \to \pm 1$.
\end{itemize}
Then the following statement is true:
\begin{align*}
    (g(x) - f(x))^2 \ge (g(x)-\ov{f}(x))^2.
\end{align*}
\end{lemma}
\begin{proof}

There are two cases to consider, since $g(x)$ is either $+1$ or $-1$:

{\bf Case 1:} Suppose $g(x) = 1$.

Then we have 
\begin{align*}
    (g(x) - f(x))^2 = (1 - f(x))^2.
\end{align*}
Since $\ov{f}(x)$ is the clipped version of $f(x)$, it satisfies:
\begin{align*}
    -1 \leq \ov{f}(x) \leq 1.
\end{align*}
Therefore we have different scenarios:
\begin{itemize}
    \item If $f(x) > 1$, then $\ov{f}(x) = 1$ and $(1 - f(x))^2 \geq 0 = (1 - \ov{f}(x))^2$.
    \item If $-1 \leq f(x) \leq 1$, then $\ov{f}(x) = f(x)$ and equality holds $(1 - f(x))^2 = (1 - \ov{f}(x))^2$.
    \item If $f(x) < -1$, then $\ov{f}(x) = -1$ and $(1 - f(x))^2 \geq 4 = (1+1)^2$.
\end{itemize}
Thus for $g(x) = 1$ the inequality always holds.

{\bf Case 2:} Suppose $g(x) = -1$.

Similarly we have 
\begin{align*}
    (g(x) - f(x))^2 = (-1 - f(x))^2.
\end{align*}
Considering the clipped function, we have
\begin{itemize}
    \item If $f(x) > 1$, then $\ov{f}(x) = 1$ and $(-1 - f(x))^2 \geq 4 = (-1 -1)^2 = (-1 - \ov{f}(x))^2$.
    \item If $-1 \leq f(x) \leq 1$, then $\ov{f}(x) = f(x)$ and $(-1 - f(x))^2 = (-1 - \ov{f}(x))^2$.
    \item If $f(x) < -1$ then $\ov{f}(x) = - 1$ and $(-1 - f(x))^2 \geq 0 = (-1+1)^2$.
\end{itemize}
Thus for $g(x) = -1$, the inequality also always holds.  
\end{proof}

\subsection{Coefficient Operator}\label{sec:append_prelim:coefficient}
We then provide the proof of the fact about the coefficient operator.
\begin{fact}[Formal version of Fact~\ref{fac:binom_to_tk} in Section~\ref{sec:prelim:coefficient}]\label{fac:binom_to_tk_append}
    Let $n,k\in\Z$ and $n\geq k$.
    Let $\mathsf{Coeff}_i[f(t)]$ be Definition~\ref{def:k_coefficient_operator}. 
    Let $t\in\R$. We have:
    \begin{align*}
        \binom{n}{k} = \mathsf{Coeff}_k[ (1+t)^n ].
    \end{align*}
\end{fact}
\begin{proof}
    By Fact~\ref{fac:binomial_thm}, we can obtain that:
    \begin{align*}
        (1+t)^n = \sum_{i=0}^n\binom{n}{i}t^i.
    \end{align*}
    Thus, we can trivially obtain by Definition~\ref{def:k_coefficient_operator} that 
    \begin{align*}
        \binom{n}{k} = \mathsf{Coeff}_k [ (1+t)^n ],
    \end{align*}
    since operator $\mathsf{Coeff}_k[\cdot]$ extracts the $k$-th coefficient of the polynomial function in it.
\end{proof}

\subsection{Calculations about Binomial Coefficients}\label{sec:append:binom_coeff}

In this section, we introduce some facts about the calculations of binomial coefficients.

\begin{lemma}[Formal version of Lemma~\ref{lem:ratio_of_card_pplus_p} in Section~\ref{sec:prelim:cal_bio}]\label{lem:append_ratio_of_card_pplus_p}
If the following conditions hold
\begin{itemize}
    \item Let $m,d,k$ be non-negative integers. 
    \item Assume that $d\geq m\geq 0.5k$ and $d-m\ge 0.5k$.
    \item Let $A := \sum_{j=0}^{k/2}\binom{m}{j}\binom{d-m}{k-j}$. 
    \item Let $B := \frac{1}{2} \binom{d}{k}$. 
    \item Let $\Gamma(t) := (1+t)^{d-m}$
    \item Let $\Delta(t):= \sum_{j=0}^{k/2} \binom{m}{j}t^j -\sum_{j=k/2+1}^m \binom{m}{j} t^j$. 
\end{itemize}
Then we can show
\begin{itemize}
\item Part 1. We have
\begin{align*}
A = \frac{1}{2} \binom{d}{k} +\frac{1}{2} \mathsf{Coeff}_{k}[\Gamma(t) \Delta(t)].
\end{align*} 
\item Part 2.
We can bound
\begin{align*}
    | \frac{A}{B} - 1 | \leq e^{-\Omega(d)}.
\end{align*}    
\end{itemize}
\end{lemma}
\begin{proof} 
{\bf Part 1: Computing $A$.}

By our definition of $\Delta(t)$, we have:
\begin{align}\label{eq:rewrite_Delta_t}
    \Delta(t) 
    = & ~ \sum_{j=0}^{k/2} \binom{m}{j} t^j -\sum_{j=k/2+1}^m \binom{m}{j} t^j \notag \\
    = & ~ 2 \sum_{j=0}^{k/2} \binom{m}{j} t^{j} -   \sum_{j=0}^m \binom{m}{j} t^j \notag \\
    = & ~ 2 \sum_{j=0}^{k/2} \binom{m}{j} t^{j} - (1+t)^m 
\end{align}
where the first step is based on the definition of $\Delta(t)$, the second step follows from canceling $\sum_{j=0}^{k/2} \binom{m}{j} t^j$ from both terms, the last step can be derived by Fact~\ref{fac:binom_to_tk}.

By Fact~\ref{fac:binom_to_tk}, we have:
 
\begin{align}\label{eq:rewrite_Gamma_t}  
    \binom{d - m}{k - j} = & ~ \mathsf{Coeff}_{k-j}[(1+t)^{d-m}] \notag \\
    = & ~ \mathsf{Coeff}_{k-j}[\Gamma(t)]
\end{align}
where the first step comes from  Fact~\ref{fac:binom_to_tk}, and the last step is derived from the definition of $\Gamma(t)$.

Then we have 
\begin{align*}
    A = & ~\sum_{j=0}^{k/2} \binom{m}{j} \binom{d-m}{k-j} \\
    = & ~ \sum_{j=0}^{k/2} \binom{m}{j} \mathsf{Coeff}_{k-j}[\Gamma(t) ]\\
    = & ~ \sum_{j=0}^{k/2} \binom{m}{j} \mathsf{Coeff}_{k}[ \Gamma(t) t^{j}]\\
    = & ~ \mathsf{Coeff}_{k}[ \Gamma(t)  \sum_{j=0}^{k/2} \binom{m}{j} t^{j}]\\
    = & ~ \frac{1}{2}\mathsf{Coeff}_{k}[ \Gamma(t)  2 \sum_{j=0}^{k/2} \binom{m}{j} t^{j}]\\
    = & ~ \frac{1}{2}\mathsf{Coeff}_{k}[ \Gamma(t) \cdot ((1+t)^m + \Delta(t) )]\\
    = & ~ \frac{1}{2}\mathsf{Coeff}_{k}[(1+t)^d + \Gamma(t) \Delta (t) ] \\
    = & ~ \frac{1}{2} \mathsf{Coeff}_{k}[(1+t)^d] + \frac{1}{2}\mathsf{Coeff}_{k}[ \Gamma(t) \Delta (t)]\\
    = & ~ \frac{1}{2} \binom{d}{k} +\frac{1}{2} \mathsf{Coeff}_{k}[ \Gamma(t) \Delta (t) ] 
\end{align*}
where the first step can be derived by the definition of $A$, the second step is based on Fact~\ref{fac:binom_to_tk}, the third step follows from Fact~\ref{fac:coeff_op_change_index}, the fourth step follows from the linearity of coefficient extraction in Fact~\ref{fac:coeff_op_lin}, the fifth step follows from simple algebra, the sixth step follows from Eq.~\eqref{eq:rewrite_Delta_t}, the seventh step follows from simple algebra, the eighth step follows from the linearity of coefficient extraction in Fact~\ref{fac:coeff_op_lin}, and the ninth step follows from Fact~\ref{fac:binom_to_tk}.

{\bf Part 2. Bounding $|\frac{A}{B} - 1|$.} 
Note that
\begin{align*}
|\frac{A}{B} - 1| = \frac{ \frac{1}{2} \binom{d}{k} + \frac{1}{2} \mathsf{Coeff}_k[\Gamma(t) \Delta(t)] }{ \frac{1}{2} \binom{d}{k}  } - 1 = \frac{\mathsf{Coeff}_k[\Gamma(t) \Delta(t)]}{\binom{d}{k}}.
\end{align*}
where the first step can be directly derived by the definition of $A$ and $B$, and the second step follows simple algebra.

Thus, we know that 
\begin{align*}
    | \frac{A}{B} - 1 | \leq e^{-\Omega(d)}.
\end{align*}
is equivalent to
\begin{align*}
| \frac{ \mathsf{Coeff}_k [ \Lambda(t) \Delta(t) ] }{ \binom{d}{k} } | \leq e^{-\Omega(d)}.
\end{align*}

Therefore, using Fact~\ref{lem:append:negligible_terms} to do the straightforward calculation of expressing details of binomial coefficients, we can conclude that 
\begin{align*}
| \frac{ \mathsf{Coeff}_k [ \Lambda(t) \Delta(t) ] }{ \binom{d}{k} } | \leq e^{-\Omega(d)},
\end{align*}
which further implies that 
\begin{align*}
    | \frac{A}{B} - 1 | \leq e^{-\Omega(d)}.
\end{align*}

This finishes the proof. 
\end{proof}
\section{Missing Proofs in Section~\ref{sec:main}}\label{sec:append_main}
Here we present the previously omitted proofs in Section~\ref{sec:main}.

\subsection{Main Theorem}\label{sec:append_main:main}

Here we introduce the formal version of our main theorem.

\begin{theorem}[Hardness of Polynomial-Sample Majority, Formal version of Theorem~\ref{thm:finite_sample_majority_2} in Section~\ref{sec:main:main}]\label{thm:finite_sample_majority_2_append}

If the following holds:
\begin{itemize}
    \item Let $c = 4c_1 + 4c_2 + 2c_3 + 2c_4 + 1$.
    \item Let $c_i > 0$ for all $i \in [4]$.
    \item $n = \Omega( d^{c} ) $.
    \item Let $k = \Theta(d)$.
    \item $\|\nabla f_{\theta}\|_2 =  O (d^{c_1} )$.
    \item Let $S \in \S$ be a fixed underlying support.
    \item Let $T= O(d^{c_3})$
\end{itemize}
Then the following statement is true:
\begin{itemize}
    \item There exists an $O (d^{-c_2} )$-approximate gradient oracle $\wt{\nabla}$
  such that with probability $1 - e^{-\Omega(d)}$ over the random choice of samples, 
  for \emph{any} (possibly randomized) iterative algorithm ${\cal A}$ making at most $T$ queries 
  to $\wt{\nabla}L_n$, the output $\theta({\cal A})$ satisfies
    \begin{align*}
       \E_{p,x} [ (\mathsf{MAJ}(x,S) - f_{\theta({\cal A})}(x) )^2 ]
     \geq  1 - O (d^{-c_4} ).
    \end{align*} 
\end{itemize}
\end{theorem}
\begin{proof}

Here for convenient of writing proofs, we create the notation $V$ to represent the a complicated quantity
\begin{align}\label{def:V:poly}
    V:= 2\sqrt{\frac{d}{n}} \sup_{\theta,x} \| \wt{\nabla} f_\theta(x)\|_2^2
\end{align}

We can give an upper bound for $V$ with the following sense
\begin{align}\label{eq:upper_bound_V:poly}
V = & ~ 2\sqrt{\frac{d}{n}} \sup_{\theta,x} \| \wt{\nabla} f_\theta(x)\|_2^2 \notag \\
\leq & ~ O(1) \sqrt{\frac{d}{ d^c } } \sup_{\theta,x} \| \wt{\nabla} f_\theta(x)\|_2^2 \notag \\
\leq & ~ O(1) \sqrt{\frac{d}{ d^c }} d^{2c_1+2c_2} \notag \\
\leq & ~ O(1) d^{2c_1 + 2c_2 + 0.5  - 0.5 c}
\end{align}
where the first step follows from definition of $V$, the second step follows from $n = \Omega( d^c)$, the third step follows from $\sup_{\theta,x} \|\nabla f_\theta(x)\|_2^2  = O(d^{2c_1})$ and $\|\wt{\nabla} f_\theta(x)\|$ is $O(d^{-c_2})$-approximation $\|\nabla f_\theta(x)\|$, and the last step follows from simple algebra.

We can show that
\begin{align}\label{eq:bound_exp_bound_c_4}
& ~ (0.5c -0.5 -2c_1- 2c_2)/2 - c_3 \notag \\
= & ~ (0.5 (4c_1+4c_2+2c_3 +2c_4 +1) -0.5 -2c_1 -2 c_2  )/2 - c_3 \notag \\
= & ~ ( 2c_3 + 2c_4 )/2 - c_3 \notag \\
\geq & ~ c_4
\end{align}
where the first step is based on choosing an appropriate $c$.

Next, we can prove
\begin{align}\label{eq:upper_bound_T_eps_d_n_sup:poly}
\frac{4T}{\epsilon^2}\sqrt{\frac{d}{n}} \sup_{\theta,x} \| \wt{\nabla} f_\theta(x)\_2^2 
= & ~ 2 T \cdot \frac{1}{\epsilon^2} \cdot V \notag \\
\leq & ~ 2 T \cdot  (V)^{-1/2} \cdot V \notag \\
= & ~ 2 T \cdot V^{1/2} \notag \\
\leq & ~ O(1) T \cdot d^{ (2c_1+2c_2+0.5 -0.5c )/2 } \notag \\
\leq & ~ O(1) \cdot d^{c_3} \cdot d^{ (2c_1+2c_2+0.5 -0.5c )/2 } \notag \\
\leq & ~ O(1)  \cdot  d^{ c_3+ (2c_1+2c_2+0.5 -0.5c )/2 }\notag \\
\leq & ~ O(1) \cdot d^{-c_4}
\end{align}
where the first step is based on the definition of $V$, the second step follows from choosing $\epsilon = V^{1/4}$, the third step can be derived by simple algebra, the fourth step follows from Eq.~\eqref{eq:upper_bound_V:poly}, the fifth step follows from $T \leq O(d^{c_3})$, the sixth step from simple algebra, and the last step follows from Eq.~\eqref{eq:bound_exp_bound_c_4}.

We can show 
\begin{align*}
    \E_{S \in \S, x}[(\mathsf{MAJ}(x,S) - f_{\theta({\cal A})}(x))^2] 
    \geq & ~ 1 - \frac{4T}{\epsilon^2}\sqrt{\frac{d}{n}} \sup_{\theta,x} \| \wt{\nabla} f_\theta(x)\|_2^2 - e^{-\Omega(d)} \\
    \geq & ~ 1 - O( d^{-c_4}) - e^{-\Omega(d)} \\
    \geq & ~ 1 - O( d^{-c_4})
\end{align*}
where the first step can be derived by Theorem~\ref{thm:mse_bound} with approximate gradient oracle $\wt{\nabla} f_{\theta}(x)$, the second step is based on Eq.~\eqref{eq:upper_bound_T_eps_d_n_sup:poly}, the last step follows from basic algebra.
\end{proof}

\begin{theorem}[Hardness of Exponential-Sample Majority, Formal version of Theorem~\ref{thm:finite_sample_majority_1} in Section~\ref{sec:main:main}]\label{thm:finite_sample_majority_1_append}

If the following conditions hold:
  \begin{itemize}
      \item Let $n = e^{\Omega(d)}$.
      \item Let $k = \Theta(d)$.
      \item Let $\|\nabla f_{\theta}\|_2 =  \poly(d)$. 
      \item Let $S \in \S$ be a fixed underlying support. 
      \item Let $T = \poly(d)$.
  \end{itemize}
Then the following statement is true:
\begin{itemize}
    \item There exists an $e^{-\Omega(d)}$-approximate gradient oracle $\wt{\nabla}$ 
  such that with probability $1 - e^{-\Omega(d)}$ over the random choice of samples, 
  for \emph{any} (possibly randomized) iterative algorithm ${\cal A}$ that makes at most $T$ 
  to $\wt{\nabla} L_n$, the output $\theta({\cal A})$ satisfies
  \begin{align*}
      \E_{S,x} [ (\mathsf{MAJ}(x,S) - f_{\theta({\cal A})}(x) )^2 ]
     \geq  1 - e^{-\Omega(d)}.
  \end{align*}
\end{itemize}
\end{theorem}
\begin{proof}

Here for convenience of writing proofs, we create the notation $V$ to represent a complicated quantity
\begin{align}\label{def:V}
    V:= 2\sqrt{\frac{d}{n}} \sup_{\theta,x} \| \wt{\nabla} f_\theta(x)\|_2^2
\end{align}

We can upper bound $V$ in the following sense
\begin{align}\label{eq:upper_bound_V}
V = & ~ 2\sqrt{\frac{d}{n}} \sup_{\theta,x} \| \wt{\nabla} f_\theta(x)\|_2^2 \notag \\
= & ~ 2 \sqrt{\frac{d}{e^{\Omega(d)} }} \sup_{\theta,x} \| \wt{\nabla} f_\theta(x)\|_2^2 \notag \\
\leq & ~ 2 \sqrt{\frac{d}{e^{\Omega(d)} }} \poly(d) \notag \\
\leq & ~ e^{-\Omega(d)}
\end{align}
where the first step is based on definition of $V$, the second step follows from $n = e^{\Omega(d)}$, the third step can be derived by $\sup_{\theta,x} \|\nabla f_\theta(x)\|_2^2  \leq \poly(d)$ and $\| \wt{\nabla} f_{\theta}(x) \|_2$ is $e^{-\Omega(d)}$-approximation to $\| \nabla f_{\theta}(x) \|_2$ , and the last step is based on simple algebra.

Next, we can prove
\begin{align}\label{eq:upper_bound_T_eps_d_n_sup}
\frac{4T}{\epsilon^2}\sqrt{\frac{d}{n}} \sup_{\theta,x} \|\nabla f_\theta(x)\|_2^2 
= & ~ 2 T \cdot \frac{1}{\epsilon^2} \cdot V \notag \\
\leq & ~ 2 T \cdot  (V)^{-2/3} \cdot V \notag \\
= & ~ 2 T \cdot V^{1/3} \notag \\
\leq & ~ 2 T \cdot e^{-\Omega(d)} \notag \\
\leq & ~ 2 \cdot \poly(d) \cdot e^{-\Omega(d)} \notag \\
\leq & ~ e^{-\Omega(d)}
\end{align}
where the first step is based on the definition of $V$, the second step chooses $\epsilon = V^{1/3}$, the third step follows from basic algebra, the fourth step follows from Eq.~\eqref{eq:upper_bound_V}, the fifth step can be derived by the fact $T \leq \poly(d)$, the last step from simple algebra.

We can show 
\begin{align*}
    \E_{S \in \S, x}[(\mathsf{MAJ}(x,S) - f_{\theta({\cal A})}(x))^2] 
    \geq & ~ 1 - \frac{4T}{\epsilon^2}\sqrt{\frac{d}{n}} \sup_{\theta,x} \|\nabla f_\theta(x)\|_2^2 - e^{-\Omega(d)} \\
    \geq & ~ 1 - e^{-\Omega(d)} - e^{-\Omega(d)} \\
    \geq & ~ 1 - e^{-\Omega(d)}
\end{align*}
where the first step follows from Theorem~\ref{thm:mse_bound}, the second step follows from Eq.~\eqref{eq:upper_bound_T_eps_d_n_sup}, the last step follows from simple algebra.

\end{proof}

\subsection{Bounding the Gradient Variance}\label{sec:append_main:gradient_variance}
Here we introduce the proof of some facts and lemmas about bounding the gradient variance.

\begin{fact}[Formal version of Fact~\ref{fac:grad_var_tool} in Section~\ref{sec:main:gradient_variance}]\label{fac:grad_var_tool_append}
    If the following conditions hold:
        \begin{itemize}
            \item Let $\S$ be the set of all the $k$-size support sets over $\{\pm 1\}^d$.
            \item Let there be $n$ fixed empirical samples $\{(x_i, y_i)\}_{i=1}^n$ as defined in Definition~\ref{def:empirical_loss}. 
            \item Let the empirical inner product $\langle \cdot, \cdot \rangle_n$ empirical norm $\| \cdot \|_n$ as defined in Definition~\ref{def:empirical_loss}. 
            \item Let $h_S:=\maj(x,S)$ denote the majority function in Definition~\ref{def:maj_func} for any $S\in\S$.
            \item Let the gram matrix be $G_{\S}:=(\langle  h_{S_1},h_{S_2}\rangle_n)_{S_1,S_2\in \S}$, in which the largest eigenvalue $\lambda_{\max}(G_{\S})$ is upper bounded. 
        \end{itemize}
    The following statement is true:
    \begin{align*}
        \sum_{S\in\S}\langle f, h_S\rangle_n^2 \leq \lambda_{\max}(G_{\S}) \|f\|_n^2.
    \end{align*}
\end{fact}
\begin{proof}
    Since we already know the largest eigenvalue of $G_{\S}$ is bounded, we can conclude that set $\{h_S:S\in \S\}$ formulates a partial frame for the empirical norm. 
    More specifically, consdering an arbitrary $f:\R^d\to\R$, we can decompose $f$ into two orthogonal components, i.e.,
    \begin{align}\label{eq:orthogonality}
    f = \sum_{S\in \S
    } \alpha_S\cdot h_S + f_0, \quad f_0\in(\mathrm{span} S)^\bot,
    \end{align}

    for some coefficient vectors $\alpha = (\alpha_S)_{S \in \S}$. 

    Thus, we have: 
    \begin{align}\label{eq:grad_var
    fact_step1}
    \| f\|_n^2 
    \geq & ~ \|f-f_0 \|_n^2 \notag\\
    = & ~ \sum_{S_1,S_2\in \S} \alpha_{S_1}\alpha_{S_2}\langle h_{S_1},h_{S_2}\rangle_n \notag \\
    = & ~ \alpha^\top G_{\S} \alpha \notag\\
    = & ~\| G_{\S}^{1/2}\alpha\|_2^2
    \end{align}
    where the first step follows from the Pythagorean theorem in the empirical norm where $\| f\|_n^2 = \| f-f_0\|_n^2 + \| f_0\|_n^2$, the second step follows from Eq.~\eqref{eq:orthogonality}, the third step is based on the definition of Gram matrix, and the last step follows the definition of vector $\ell_2$ norm.
    
    Therefore, we can obtain the desired result as follows:  
    \begin{align*}
    \sum_{S\in \S}\langle f,h_{S}\rangle_n^2 
    = & ~\sum_{S_1\in \S}(\sum_{S_2\in \S} \alpha_{S_2}\langle h_{S_2},h_{S_1}\rangle_n)^2 \notag \\
    = & ~ \sum_{S_1\in \S} ( \langle \alpha, (G_{\cal S})_{S_1,*} \rangle )^2 \notag \\
    = & ~ \|G_{\S} \alpha\|_2^2 \notag \\
        \le & ~ \| G_{\S}^{1/2}\|^2 \| G_{\S}^{1/2} \alpha\|^2_2\notag \\
    = & ~ \lambda_{\max}(G_{\S}) \| G_{\S}^{1/2} \alpha\|^2_2 \notag \\
    \leq & ~ \lambda_{\max}(G_{\S}) \|f\|_n^2.
    \end{align*}
    where the first step can be derived by Eq.~\eqref{eq:orthogonality} and the linearity of the inner product, the second step follows from the definition of the Gram matrix, the third step follows from the definition of $\| \cdot \|_2$, the forth step follows basic matrix analysis, the fifth step follows from the definition of matrix spectral norm, and the last step is based on Eq.~\eqref{eq:grad_var
    fact_step1}.

    This finishes the proof.
\end{proof}

\begin{lemma}[Bound on Gradient Variance for Majority Learning, Formal version of Lemma~\ref{lem:grad_var} in Section~\ref{sec:main:gradient_variance}]\label{lem:grad_var_append}
    Let $\mathrm{Var}_n(\theta; \S)$ be the gradient variance as defined in Definition~\ref{def:grad_var}. Then, with probability at least $1 - e^{-\Omega(d)}$ over random sampling, the variance of the gradient of the empirical loss with respect to the target majority is bounded as:
    \begin{align*}
        \mathrm{Var}_n(\theta; \S) 
        \leq   2\sqrt{\frac{d}{n}} \sup_{\theta,x} \|\nabla f_\theta(x)\|_2^2.
    \end{align*}
\end{lemma}

\begin{proof}
{\bf Part 1. Bounding Each Element. }
     By Definition~\ref{def:grad_var}, we consider the variance of the empirical gradient $\nabla L_{n,S}$ w.r.t. the target underlying support $S$: 
    \begin{align}\label{eq:variance}
        \mathrm{Var}_n(\theta; \S) = \E_{S_1\in \S}[\|\nabla L_{n,S_1}(\theta) - \E_{S_2\in \S}[\nabla L_{n,S_2}(\theta)]\|_2^2].
    \end{align}
    
    For notation simplicity, we define function $h_{S}(\cdot):=\mathsf{MAJ}(x,S)$ for a fixed support $S\in \S$. We proceed to evaluate the magnitude of $\mathrm{Var}_n(\theta;\S)$. For $S_1,S_2\in \S$ with $S_1\neq S_2$ it holds that
    \begin{align*}
        \langle h_{S_1},h_{S_2}\rangle_n = & ~ \frac{1}{n} \sum_{i=1}^n (\mathsf{sgn}(\sum_{j_1\in S_1} x_{i,j_1})\cdot \mathsf{sgn}(\sum_{j_2\in S_2} x_{i,j_2})), 
    \end{align*}
    which directly follows from Definition~\ref{def:maj_prob} and the definition of empirical norm from Definition~\ref{def:empirical_loss}.
    
    Since $a_i:=\mathsf{sgn}(\sum_{j_1\in S_1} x_{i,j_1})\cdot \mathsf{sgn}(\sum_{j_2\in S_2} x_{i,j_2})$ is i.i.d. $\mathrm{Unif}(\{\pm 1\})$ for fixed $S_1, S_2$, by applying a union bound over Hoeffding's inequality from Lemma~\ref{lem:hoeffding}, it follows for $\delta:= 2\sqrt{\frac{d}{n}}$ that
    \begin{align*}
        \Pr[\sup_{S_1\neq S_2}|\langle h_{S_1},h_{S_2}\rangle_n| \geq \delta] 
        \le & ~ |\S|^2 \exp(-\frac{n\delta^2}{2}) \\ \le & ~ \binom{d}{k}^2 e^{-2d} \\
        \le & ~ (\frac{2}{e})^{2d},
    \end{align*}
    where the first step follows from Hoeffding's inequality from Lemma~\ref{lem:hoeffding}, the second step follows from Definition~\ref{def:maj_prob} where $|\S| = \binom{d}{k}$, and the third step follows from Fact~\ref{fac:cardinality}. 
    
     {\bf Part 2: Bounding the Largest Eigenvalue of the Gram Matrix.} 
    Then with probability at least $1-e^{-\Omega(d)}$ over random sampling, every off-diagonal component of the Gram matrix can be defined as:
    \begin{align}\label{eq:gram_matrix}
        G_{\S}:=(\langle  h_{S_1},h_{S_2}\rangle_n)_{S_1,S_2\in \S},
    \end{align}
     which has magnitude at most $\delta$, while the diagonal entries are equal to $1$. 

     Thus, we have: 
     \begin{align*}
         |\lambda_{\max}(G_{\S})| = & ~ \|G_{\S}\| \\
        \leq & ~ \|G_{\S} \|_F \\ 
         = & ~ (\sum_{i=1}^{|\S|}\sum_{j=1}^{|\S|} (G_{\S})_{i,j}^2)^{-0.5} \\ 
        \leq & ~ (|\S|^2\delta^2)^{-0.5} \\
         \leq &~ |\S|\delta,
     \end{align*}
     where basic algebra supports the first and the second steps, the third step is based on the definition of the Frobenius norm, and the fourth and final steps follow from basic algebra.

    Thus, we have $\lambda_{\max}(G_{\S})\le \delta|\S|$. 
    
    {\bf Part 3. Bounding the Variance.} 
    Let $D:=\dim \theta$ be the dimension of the model parameters. We can therefore bound $\mathrm{Var}_n(\theta; \S)$ as  
    \begin{align*}
        \mathrm{Var}_n(\theta; \S) = & ~\inf_{\mu\in\R^D} \E_{S\in \S}[\|\nabla L_{n,S}(\theta) - \mu\|_2^2]\\
        \le & ~ \E_{S\in \S}[\| \nabla L_{n,S}(\theta) - \frac{1}{n}\sum_{i=1}^n f_\theta(x_i) \nabla f_\theta(x_i)\|_2^2]\\
        = & ~ \E_{S\in \S}[\| \nabla_{\theta}(\frac{1}{2}\| h_S - f_{\theta} \|_n^2) - \frac{1}{n}\sum_{i=1}^n f_\theta(x_i) \nabla f_\theta(x_i)\|_2^2]\\ 
        = & ~ \E_{S\in \S}[\|\frac{1}{n}\sum_{i=1}^n (f_\theta(x_i) -h_{S}(x_i)) \nabla f_\theta(x_i) - \frac{1}{n}\sum_{i=1}^n f_\theta(x_i) \nabla f_\theta(x_i)\|_2^2]\\
        = & ~ \E_{S\in \S}[\|\frac{1}{n}\sum_{i=1}^n h_{S}(x_i)\nabla f_\theta(x_i)\|_2^2]\\
        = & ~ \E_{S\in \S}[\sum_{j=1}^D (\frac{1}{n}\sum_{i=1}^n h_{S}(x_i)\nabla f_{\theta_j}(x_i))^2]\\
        = &~ \E_{S\in \S}[\sum_{j=1}^D \langle \nabla_{\theta_j}f_\theta, h_{S}\rangle_n^2]\\
        = &~ \frac{1}{|\S|} \sum_{S\in \S} \sum_{j=1}^D \langle \nabla_{\theta_j}f_\theta, h_{S}\rangle_n^2\\
        \le & ~ \sum_{j=1}^D \frac{\lambda_{\max}(G_{\S})}{|\S|} \|\nabla_{\theta_j}f_\theta\|_n^2\\
        \le & ~ 2\sqrt{\frac{d}{n}} \sup_{\theta,x} \|\nabla f_\theta(x)\|_2^2.
    \end{align*}
    where the first step follows from Eq.~\eqref{eq:variance}, the second step follows from selecting $\mu := \frac{1}{n}\sum_{i=1}^n f_\theta(x_i) \nabla f_\theta(x_i)$, the third step follows from the definition of empirical loss in Definition~\ref{def:empirical_loss}, the fourth step follows from taking the derivative w.r.t. the empirical loss, the fifth step and the sixth step follows from basic algebra, the seventh step follows from the definition of empirical norm in Definition~\ref{def:empirical_loss}, the eighth step follows from the fact that all the support sets in $\S$ is sampled uniformly in Definition~\ref{def:maj_prob}, the ninth step follows from Fact~\ref{fac:grad_var_tool}, and the last step is based on Part 2 of this proof which chooses $\delta = 2\sqrt{\frac{d}{n}}$.
    
Thus, we complete the proof.
\end{proof}

\subsection{Property of the Oracle}\label{sec:append_main:oracle}
In this section, we introduce the proof of lemma about the property of the oracle.
\begin{lemma}[Guarantee of the Oracle, Formal version of Lemma~\ref{lem:privacy_oracle} in Section~\ref{sec:main:oracle}]\label{lem:privacy_oracle_append}
If the following conditions hold:
\begin{itemize}
    \item Let $Q\subseteq \S$ denote the subset of the hypothesis space where the oracle defaults to the mean gradient.
    \item Let $T$ be the number of optimization steps.
    \item Let $d$ be the dimensionality of the input. 
    \item Let $n$ be the sample size.
    \item Let $f_\theta(x)$ be the model's prediction function.
\end{itemize}
Then the following statements are true:

\begin{itemize}
    \item During $T$ optimization steps, the oracle reveals no information  
    about the true support $S\in\S$ with probability at least:
    \begin{align}
\Pr[Q] \ge 1 - \frac{2T}{\epsilon^2}\sqrt{\frac{d}{n}} \sup_{\theta,x} \|\nabla f_\theta(x)\|_2^2,
\end{align}
\end{itemize} 
\end{lemma}
\begin{proof}
    We first analyze when the oracle reveals information about the true support $S$.

The oracle substitutes the true gradient $\nabla L_{n,S}(\theta)$ with the expected gradient $\E_{S_2\in \S}[\nabla L_{n,S_2}(\theta)]$ whenever the difference exceeds threshold $\epsilon$. Information about the true support $S$ is revealed only when the oracle does not default to the mean gradient, which occurs when:
\begin{align*}
\|\nabla L_{n,S_1}(\theta) - \E_{S_2\in \S}[\nabla L_{n,S_2}(\theta)]\| \leq \epsilon.
\end{align*}
We bound the probability of this occurrence for a single optimization step.

Let $X$ be the event that $\|\nabla L_{n,S_1}(\theta) - \E_{S_2\in \S}[\nabla L_{n,S_2}(\theta)]\| \leq \epsilon$. By applying concentration inequalities and considering the variance of the gradient estimates, we can follow from Corollary~\ref{cor:grad_var} to show that:
\begin{align}\label{eq:prob_x_bound}
\Pr[X] \leq 2\epsilon^{-2}\sqrt{\frac{d}{n}} \sup_{\theta,x} \|\nabla f_\theta(x)\|_2^2.
\end{align}

Then we apply a union bound across all $T$ optimization steps. Let $X_t$ be the event that the oracle reveals information about $S$ at step $t$. We denote the event such that the oracle reveals information in any step by $\Psi$. Then by the union bound, we have the following:
\begin{align*}
\Pr[\Psi] = & ~ \Pr[\bigcup_{t=1}^T X_t] \\
\leq & ~ \sum_{t=1}^T \Pr[X_t] \\
= & ~ \frac{2T}{\epsilon^2}\sqrt{\frac{d}{n}} \sup_{\theta,x} \|\nabla f_\theta(x)\|_2^2,
\end{align*}
where the first step comes from basic probability, the second step applies the union bound, and the last step follows from Eq.~\eqref{eq:prob_x_bound}.

Therefore, since event $\Psi$ is the complement of event $Q$, we can conclude by basic probability that:
\begin{align*}
    \Pr[Q] = & ~ 1 - \Pr[\Psi]\\
\geq & ~ 1 - \frac{2T}{\epsilon^2}\sqrt{\frac{d}{n}} \sup_{\theta,x} \|\nabla f_\theta(x)\|_2^2. 
\end{align*}

This completes the proof.

\end{proof}

\subsection{Lower Bounding the \texorpdfstring{$L_{\infty}$}{} Loss}\label{sec:append_main:l_infty_loss}
In this section, we provide all the details of the proof about Theorem~\ref{thm:lower_l_infty_error_append}.

\begin{theorem}[Lower Bound on $L_\infty$ Error, Formal version of Theorem~\ref{thm:lower_l_infty_error} in Section~\ref{sec:main:l_infty_loss}]\label{thm:lower_l_infty_error_append}
If the following conditions hold:
\begin{itemize}
    \item Let $Q \subseteq \S$ be a noninformative set with an automorphism $\sigma: Q \to Q$ having no fixed points. 
    \item Let function $f $ be $f_{\theta({\cal A})}: \{\pm 1\}^d \to \R$. 
\end{itemize}
Then the following statement is true:
\begin{align*}
\E_{S \in \S}[\sup_{x}|\mathsf{MAJ}(x, S) - f_{\theta({\cal A})}(x)|] \geq \Pr[S \in Q].
\end{align*}
\end{theorem}
\begin{proof}
 {\bf Part 1: Restricting to the Noninformative Set.} Let $h_{S}(x):=\mathsf{MAJ}(x,S)$. 
We can bound the $L_\infty$ error by restricting to the noninformative set $Q$ as follows: 
\begin{align} \label{eq:loss_infty_step1}
 & \E_{S_1 \in \S} [\sup_{x}|h_{S_1}(x) - f_{\theta({\cal A})}(x)| ] \notag\\
\geq & ~  \E_{S_1 \in \S} [\mathbf{1}[S_1 \in Q] \cdot \sup_{x}|h_{S_1}(x) - f(x)| ]\notag \\
= & ~ \frac{1}{|\S|}\sum_{S_1\in Q} ( \sup_{x}|h_{S_1}(x) - f(x)| ) \notag \\
= & ~ \frac{1}{2|\S|} \sum_{S_1 \in Q} \underbrace{(\sup_{x} |h_{S_1}(x) - f(x)| + \sup_{x} |\sigma \circ h_{S_1}(x) - f(x)| )}_{:=T(S_1)},
\end{align}
where the first step follows from it is possible that the indicator function $\mathbf{1}[S_1 \in Q]$ will erase some non-negative terms to zero, the second step follows from the fact that all the support sets in $\S$ are sampled from a uniform distribution as mentioned in Definition~\ref{def:maj_prob}, and the last step is based on the fact that $\sum_{S_1\in Q}\sup_{x} |h_{S_1}(x) - f(x)| = \sum_{S_1\in Q}\sup_{x} |\sigma \circ h_{S_1}(x) - f(x)|$. 

 {\bf Part 2: Bounding Each Term $T(S_1)$.} 
Then for each $S \in Q$, let $x_0 \in \{\pm 1\}^d$ be one specific binary vector such that $h_{S_1}(x_0) \neq \sigma \circ h_{S_1}(x_0)$. We can lower bound each term $T(S_1)$ in Eq.~\eqref{eq:loss_infty_step1} by selecting a specific $x_0$ from all the possible choices of $x$:
\begin{align}\label{eq:loss_infty_step2}
T(S_1) = & ~ 
\sup_{x} |h_{S_1}(x) - f(x)| + \sup_{x} |\sigma \circ h_{S_1}(x) - f(x)| \notag \\
\geq & ~ |h_{S_1}(x_0) - f(x_0)| + |\sigma \circ h_{S_1}(x_0) - f(x_0)|,
\end{align}
where the first step follows from the definition of $T(S_1)$ in Eq.~\eqref{eq:loss_infty_step1}, and the second step follows from basic algebra.

Since $h_{S_1}(x_0) \neq \sigma \circ h_{S_1}(x_0)$, without loss of generality, we can set $h_{S_1}(x_0) = 1$ and $\sigma\circ h_{S_1}(x_0)=-1$ and obtain the following:
\begin{align}\label{eq:loss_infty_step3}
    T(S_1) \geq & ~ |1 - f(x_0)| + |-1 - f(x_0)| \notag \\
= & ~ |1 - f(x_0)| + |f(x_0) + 1| \notag \\
\geq & ~ 2,
\end{align}
where the first step follows from Eq.~\eqref{eq:loss_infty_step2}, the second step comes from basic algebra, and the last step is based on the triangle inequality. 

 {\bf Part 3: Obtaining the Final Lower Bound.} Therefore, we can combine the previous parts of the proof and get the desired bound as follows:
\begin{align*}
 \E_{S_1 \in \S} [\sup_{x}|h_{S_1}(x) - f_{\theta({\cal A})}(x)| ] \ge & ~ \frac{1}{2|\S|}\sum_{p\in Q} T(S_1) \\
 \ge & ~ \frac{1}{2|\S|} \cdot 2|Q| \\
 = & ~\frac{|Q|}{|\S|} \\
 = & ~\Pr[S \in Q],
\end{align*}
where the first step follows from Eq.~\eqref{eq:loss_infty_step1}, the second step follows from Eq.~\eqref{eq:loss_infty_step2}, the third and the last steps follow from basic algebra.
\end{proof}

\subsection{Lower Bounding the Mean Squared Error}\label{sec:append_main:mse}
In this section, we introduce the proofs of some important results.
\begin{lemma}[Formal version of Lemma~\ref{lem:ppp} in Section~\ref{sec:main:mse}]\label{lem:ppp_append}
    If $k=\Theta(d)$, it satisfies with probability at least $1-e^{-\Omega(d)}$ over random sampling that
    \begin{align*}
    |\E_{S\in \S}[\maj(x, S)]| \le e^{-\Omega(d)}.
    \end{align*}
\end{lemma}

\begin{proof}
 {\bf Part 1: High Probability Result for $m$ Being Close to $\frac{d}{2}$.} We denote the number of negative ones in $x$ by $m$. By the Chernoff bound from Lemma~\ref{lem:chernoff} for the binomial distribution,
\begin{align*}
\Pr[|m - \frac{d}{2}| \le \frac{\delta d}{2}] \ge 1 - 2\exp(-\frac{\delta^2 d}{6})
\end{align*}
for a constant $\delta\in (0,1)$ to be determined, so we assume the above event throughout the proof. Let $h_S(x):=\mathsf{MAJ}(x,S)$. Moreover denoting the complement support $S^c = [d]\setminus S$, it holds that $h_S(x) = x_1\cdots x_d \cdot h_{S^c}(x)$ and $|\E_{S\in \S}[h_S(x)]| = |\E_{S\in \S}[h_{S^c}(x)]|$, so it suffices to consider the case where $2k \le d$.

 {\bf Part 2: Upper Bounding $\frac{|\S_+|}{|\S|}$.} Let $\S_+:=\{S\in\S: \maj(x,S)=1\}$ and $m$ be the number of negative ones in $x$. We consider three cases:

{\bf Case 1: $d-m < 0.5k$.} In this case, the number of $+1$ in the input $x$ is sufficiently small, which ensures that $h_S(x)=-1$ for all $x\in\{\pm1\}^d$ and $S\in\S$. Thus, we have 

{\bf Case 2: $m < 0.5k$.} In this case, the number of $-1$ in the input $x$ is sufficiently small, which ensures that $h_S(x)=1$ for all $x\in\{\pm1\}^d$. 

{\bf Case 3: $d-m \geq 0.5k, m\geq 0.5k$. } In this case, not all $x\in\{\pm1\}^d$ have the same result $h_S(x)$. First, by basic combinatorics, we can follow Definition~\ref{def:maj_prob} and obtain that 
\begin{align*}
    |\S| = \binom{d}{k}.
\end{align*}

Besides, if we consider choosing less than $k/2$ elements from $[m]$ and choosing the remaining elements from $[d] \setminus [m]$, we can make sure that $h_S(x)=1$ holds. Recall that we assume there are $m$ numbers of $-1$ in the input $x$ and there are $d-m$ numbers of $1$ in the input $x$. Thus, we have:
\begin{align*}
    |\S_+| = & ~ \sum_{j=0}^{k/2} \binom{m}{j}\binom{d-m}{k-j}. 
\end{align*}

We also have:
\begin{align*}
    |\S| = & ~ \binom{d}{k}. 
\end{align*}

Therefore, we can conclude by Lemma~\ref{lem:ratio_of_card_pplus_p} that 
\begin{align*}
    |\frac{|\S_+|}{|\S|} - 1| \leq e^{-\Omega(d)}.
\end{align*}

{\bf Part 3: The Final Bound.} From this, we conclude that
\begin{align*}
|\E_{S\in \S}[h_S(x)]| = |\frac{|\S \setminus \S_+|-|\S_+|}{|\S|}| \le e^{-\Omega(d)}
\end{align*}
with probability $1-e^{-\Omega(d)}$.  

\end{proof}

\begin{theorem}[Lower Bound on Mean Squared Error, Formal version of Theorem~\ref{thm:mse_bound} in Section~\ref{sec:main:mse}]\label{thm:mse_bound_append}
If the following conditions hold:
\begin{itemize}
     \item We have a function $f_{\theta({\cal A})}: \{\pm 1\}^d \to \R$ learned by algorithm ${\cal A}$ using $n$ samples. 
     \item Let $\ov{f}_{\theta({\cal A})}:=\min\{\max\{f_{\theta({\cal A})}(x), -1\}, 1\}$ denote the clipped version of function $f_{\theta({\cal A})}$.
    \item Let $Q \subseteq \S$ be a noninformative set.
    \item We denote the number of iterations in optimization by $T$.
    \item Let $\epsilon > 0$ denote an arbitrary gap between a single gradient and the expectation of gradient as defined in Corollary~\ref{cor:grad_var}.  
\end{itemize}
Then the following statement is true:

\begin{align*}
\E_{S \in \S, x}[(\mathsf{MAJ}(x,S) - f_{\theta({\cal A})}(x))^2] \geq 1 - \frac{4T}{\epsilon^2}\sqrt{\frac{d}{n}} \sup_{\theta,x} \|\nabla f_\theta(x)\|_2^2 - e^{-\Omega(d)}.
\end{align*} 
\end{theorem}

\begin{proof}
{\bf Part 1. Restricting to the Noninformative Set.} For notation simplicity, we let $h(x):=\mathsf{MAJ}(x,S)$, and denote functions $f_{\theta({\cal A})}$ and $\ov{f}_{\theta({\cal A})}$ by $f$ and $\ov{f}$. We first restrict the mean squared error to the noninformative set $Q$:
\begin{align}\label{eq:mse_lb_step1}
& ~ \E_{S \in \S, x}[(h_S(x) - f(x))^2]\\ 
\geq & ~ \E_{S \in \S, x}[\mathbf{1}[S \in Q] (h_S(x) - f(x))^2] \notag \\ 
\geq & ~ \E_{S \in \S, x}[\mathbf{1}[S \in Q] (h_S(x) - \ov{f}(x))^2] \notag \\
= & ~ \E_{S \in \S, x}[\mathbf{1}[S \in Q] h_S(x)^2] - 2\E_{S \in \S, x}[\mathbf{1}[S \in Q] h_S(x)\ov{f}(x)] + \E_{S \in \S, x}[\mathbf{1}[S \in Q] \ov{f}(x)^2] \notag \\
= & ~ \underbrace{\Pr[Q]}_{:=E_1} - 2\underbrace{\E_{S \in \S, x}[\mathbf{1}[S \in Q] h_S(x)\ov{f}(x)]}_{:=E_2} + \underbrace{\Pr[Q] \cdot \E_{x}[\ov{f}(x)^2]}_{:=E_3},
\end{align}
where the first step follows from erasing some non-negative terms to zero by the indicator function, the second step follows from Lemma~\ref{lem:clipping}, the third step follows from basic algebra, and the last step follows from basic probability. 

{\bf Part 2. Bounding the term $E_2$.} Specifically, we consider lower bounding the second expectation $E_2$ in Eq.~\eqref{eq:mse_lb_step1}. We can first recall Lemma~\ref{lem:ppp} to conclude that with probability $1 - e^{-\Omega(d)}$ over the sampling of $x$, we have
\begin{align*}
    \E_{S \in \S}[h_S(x)]| \leq e^{-\Omega(d)}.
\end{align*}

By basic probability, this directly implies:
\begin{align}\label{eq:mse_lb_step2}
|\E_{S \in \S, x}[h_S(x)\ov{f}(x)]| 
\leq & ~ (1 - e^{-\Omega(d)})|\E_{S \in \S}[h_S(x)]| \cdot |\E_x[\ov{f}(x)]| + e^{-\Omega(d)} \notag\\
\leq & ~ e^{-\Omega(d)}.
\end{align}

Therefore, the desired lower bound can be obtained as follows:
\begin{align}\label{eq:mse_lb_step3}
    E_2 = & ~  \E_{S \in \S, x}[\mathbf{1}[S \in Q] h_S(x)\ov{f}(x)] \notag \\ 
    = & ~ \E_{S \in \S, x}[h_S(x)\ov{f}(x)] - \E_{S \in \S, x}[\mathbf{1}[\{p \notin Q\}] h_S(x)\ov{f}(x)] \notag \\ 
    \leq & ~ e^{-\Omega(d)} + (1 - \Pr[Q])\E_{x}[|\ov{f}(x)|] \notag \\
    \leq & ~ e^{-\Omega(d)} + (1 - \Pr[Q])\sqrt{\E_{x}[1^2] \cdot \E_{x}[\ov{f}(x)^2]} \notag \\
    = & ~ e^{-\Omega(d)} + (1 - \Pr[Q])\sqrt{\E_{x}[\ov{f}(x)^2]}  \notag\\
    \leq & ~ e^{-\Omega(d)} + \frac{(1 - \Pr[Q])^2}{2\Pr[Q]} + \frac{\Pr[Q]}{2}\E_{x}[\ov{f}(x)^2],
\end{align}
where the first step follows from the definition of $E_2$, the second step follows from basic probability, the third step follows from Eq,~\eqref{eq:mse_lb_step2} and basic probability, the fourth step follows from the definition of expectation, the fifth step follows from basic algebra, and the last step follows from the Cauchy-Schwarz inequality.

{\bf Part 3. Bounding the Mean Squared Error.} Now we can recall Eq.~\eqref{eq:mse_lb_step1} in the first part of this proof, and then bound the mean squared error as follows:
\begin{align*}
& ~ \E_{S \in \S, x}[(h_S(x) - f_{\theta({\cal A})}(x))^2] \\
\geq & ~ E_1 - 2E_2 + E_3 \\
\geq & ~ E_1 - 2 (e^{-\Omega(d)} + \frac{(1 - \Pr[Q])^2}{2\Pr[Q]} + \frac{\Pr[Q]}{2}\E_{x}[\ov{f}(x)^2] ) + E_3 \\
= & ~ \Pr[Q] - 2 (e^{-\Omega(d)} + \frac{(1 - \Pr[Q])^2}{2\Pr[Q]} + \frac{\Pr[Q]}{2}\E_{x}[\ov{f}(x)^2] ) + \Pr[Q] \cdot \E_{x}[\ov{f}(x)^2] \\
= & ~ \Pr[Q] - \frac{(1 - \Pr[Q])^2}{\Pr[Q]} - 2e^{-\Omega(d)} \\
= & ~ 2 - \frac{1}{\Pr[Q]} - 2e^{-\Omega(d)} \\
\geq & ~ 1 - 2(1 - \Pr[Q]) - 2e^{-\Omega(d)} \\
= & ~ 1 - 2(1 - \Pr[Q]) - e^{-\Omega(d)},
\end{align*}
where the first step follows from Eq.~\eqref{eq:mse_lb_step1}, the second step follows from the substituting $E_2$ with its upper bound in Eq~\eqref{eq:mse_lb_step3}, the third step follows from the definition of $E_1$ and $E_2$ in Eq.~\eqref{eq:mse_lb_step1}, the fourth and the fifth steps follow from basic algebra, the sixth step follows from the fact that $2 - (1-p)^{-1}\geq 1-2t, \forall x\in [0,0.5]$, and the last step follows from absorbing the coefficient in $2e^{-\Omega(d)}$ into $-\Omega(d)$.

Therefore, we can apply the lower bound of $\Pr[Q]$ in Lemma~\ref{lem:privacy_oracle}, and obtain the final result:
\begin{align*}
\E_{S \in \S, x}[(h_S(x) - f_{\theta({\cal A})}(x))^2] 
\geq 1 - \frac{4T}{\epsilon^2}\sqrt{\frac{d}{n}} \sup_{\theta,x} \|\nabla f_\theta(x)\|_2^2 - e^{-\Omega(d)}.
\end{align*}
Thus we complete the proof.
\end{proof}

\section{Bounding the Binomial Coefficient}\label{sec:bound_binomial_coefficient}

This subsection aims to establish a technical lemma related to binomial coefficients.

\begin{lemma}\label{lem:append:negligible_terms}
    If the following conditions hold:
    \begin{itemize}
        \item Let $m,d,k$ be non-negative integers.
        \item We assume that $d\geq m\geq 0.5k$ and $d-m\ge 0.5k$. 
        \item Let $\Gamma(t) := (1+t)^{d-m}$. 
        \item Let $\Delta(t) := \sum_{j=0}^{k/2} \binom{m}{j} t^j -\sum_{j=k/2+1}^m \binom{m}{j} t^j$. 
    \end{itemize}
    The following statement is true:
    \begin{align*}
        | \frac{ \mathsf{Coeff}_{k}[\Gamma(t) \Delta (t) ] }{ \binom{d}{k} } | < e^{-\Omega(d)}.
    \end{align*}
\end{lemma}

\begin{proof}
Recall two definitions
    \begin{align*}
        \Gamma(t) := (1 + t)^{d-m}, \quad \Delta(t) := \sum_{j = 0}^{k/2} \binom{m}{j} t^j - \sum_{j = k/2 + 1}^{m} \binom{m}{j} t^j.
    \end{align*}

    Then we have
    \begin{align}\label{eq:rewrite_coeff_f_Delta_1}
        \mathsf{Coeff}_k [ \Gamma(t) \Delta(t) ] 
        = & ~ \sum_{i=0}^{k} \mathsf{Coeff}_i[ \Gamma(t) ] \cdot \mathsf{Coeff}_{k-i}[\Delta(t)] \notag \\
        = & ~ \sum_{i=0}^{k}\binom{d - m}{i} \cdot \delta_{k - i},
    \end{align}
    where the first step follows definition of $\mathsf{Coeff}$, and the last step from definition $\Gamma(t)$ and 
    where $\delta_j = \mathsf{Coeff}_j[\Delta(t)]$, which is
    \begin{itemize}
        \item $\delta_j = \binom{m}{j}$ for $j \leq k/2$.
        \item $\delta_j = -\binom{m}{j}$ for $j > k/2$,
        \item $\delta_j = 0$ otherwise.
    \end{itemize}
    So we can further rewrite $\mathsf{Coeff}_k [ \Gamma(t) \Delta(t) ]$:
    \begin{align}\label{eq:rewrite_coeff_f_Delta_2}
        \mathsf{Coeff}_k [ \Gamma(t) \Delta (t)] 
        = & ~ \sum_{i=0}^{k}\binom{d - m}{i} \cdot \delta_{k - i} \notag \\
        = & ~ \sum_{j=0}^m \binom{d-m}{k-j}\cdot \delta_j \notag \\
        = & ~ \sum_{j=0}^{k/2} \binom{d - m}{k - j} \binom{m}{j} - \sum_{j=k/2 + 1}^{m} \binom{d - m}{k - j} \binom{m}{j}.
    \end{align}
    where the first step follows from Eq.~\eqref{eq:rewrite_coeff_f_Delta_1}, the second step follows from renaming the summation index, the last step follows from definition of $\delta_j$.
    
    Note that the result in Eq.~\eqref{eq:rewrite_coeff_f_Delta_2} is close to a binomial convolution, but with alternating weights. If we replaced $\Delta(t)$ by $(1 + t)^m$, we would get:
    \begin{align*}
        (1 + t)^d = \underbrace{ (1 + t)^{d - m} }_{\Gamma(t)} (1 + t)^m.
    \end{align*}
    So $ \mathsf{Coeff}_k[(1 + t)^d] = \sum_{j=0}^{m} \binom{d - m}{k - j} \binom{m}{j} = \binom{d}{k}$. Therefore, $\Delta(t)$ is a signed approximation to $(1 + t)^m$, emphasizing low-degree terms and subtracting high-degree ones. Hence, $(1 + t)^{d - m} \Delta(t)$ is a truncated or ''tilted" version of $(1 + t)^d$, designed to cancel much of the mass around the center.

    Since $\Delta(t)$ subtracts the tail of the binomial $(1 + t)^m$, when we convolve with $(1 + t)^{d - m}$, the resulting coefficient at position $k$ will miss a significant amount of the mass from the "usual" $\binom{d}{k}$, and due to cancellation between positive and negative parts, the residual will be very small.

    More formally, define the random variable $X \sim \mathrm{Binomial}(d, 1/2)$, and consider the convolution:
    \begin{align*}
        \Pr[X = k] = \binom{d}{k} 2^{-d}
    \end{align*}
    and numerator is a signed combination over subsets. The signed sum in numerator can be shown to correspond to:
    \begin{align*}
        \Pr[X = k \mathrm{~and~} Y \leq k/2] - \Pr[X = k \mathrm{~and~} Y > k/2],
    \end{align*}
    where $X = A + B$, $A \sim \mathrm{Binomial}(d - m, 1/2)$, $B \sim \mathrm{Binomial}(m, 1/2)$. The probabilities where $B$ is large (i.e., $> k/2$) cancel out the corresponding terms in the total $\Pr[X = k]$, leaving behind a small residue.
    
    This residual is upper-bounded by the tail probability of $B \sim \mathrm{Binomial}(m, 1/2)$ deviating from its mean $m/2$ by $\Omega(m)$, which is known to decay as $e^{-\Omega(m)}$. Since $m = \Theta(d)$, this gives:
    \begin{align*}
        |S| < \binom{d}{k} \cdot e^{-\Omega(d)},
    \end{align*}
    and therefore we can get:
    \begin{align*}
        | \frac{ \mathsf{Coeff}_k[\Gamma(t)\Delta(t)] }{ \binom{d}{k} } | < e^{-\Omega(d)}.
    \end{align*}
    Thus, we finally complete the proof.
\end{proof}




\end{document}